\documentclass{ieeetmlcn}
\usepackage{cite}
\usepackage{amssymb}
\usepackage{algorithmic}
\usepackage{graphicx,color}
\usepackage{textcomp}
\usepackage{xcolor}
\usepackage{hyperref}
\def\BibTeX{{\rm B\kern-.05em{\sc i\kern-.025em b}\kern-.08em
    T\kern-.1667em\lower.7ex\hbox{E}\kern-.125emX}}
\AtBeginDocument{\definecolor{tmlcncolor}{cmyk}{0.93,0.59,0.15,0.02}\definecolor{NavyBlue}{RGB}{0,86,125}}
\usepackage{tabularx}
\usepackage{array}

\usepackage{amsmath,amsfonts,bm,bbm}
\allowdisplaybreaks
\usepackage[mathscr]{euscript}
\usepackage{color, soul}
\usepackage{algorithm,algorithmic,multirow,hhline}
\usepackage{dblfloatfix}
\usepackage{mathalfa} 
\usepackage{caption}
\usepackage{comment}
\usepackage{subfigure}
\usepackage{pstricks}
\usepackage{subfloat}
\usepackage{setspace}
\usepackage{booktabs}
\usepackage{multicol}
\usepackage{mathtools}

\newcolumntype{P}[1]{>{\centering\arraybackslash}p{#1}}
\newtheorem{theorem}{Theorem}
\newtheorem{lemma}{Lemma}

\newtheorem{definition}{Definition}

\newtheorem{assumption}{Assumption}

\providecommand{\propositionname}{Proposition}

\def\OJlogo{\vspace{-4pt}\includegraphics[height=0pt]{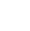}}
\def\seclogo{\vspace{10pt}\includegraphics[height=0pt]{new_logo}}

\begin{document}

\markboth{Distributionally Robust Federated Learning with Client Drift Minimization}{M. Krouka et al.}

\title{Distributionally Robust Federated Learning with Client Drift Minimization}

\author{MOUNSSIF KROUKA, CHAOUKI BEN ISSAID,\\ and MEHDI BENNIS (Fellow, IEEE)}
\corresp{Centre for Wireless Communications (CWC), University of Oulu, 90570 Oulu, Finland \vspace{0.1cm} \\ 
Corresponding author: M. KROUKA (email: mounssif.krouka@oulu.fi)}

\begin{abstract}
Federated learning (FL) faces critical challenges, particularly in heterogeneous environments where non-independent and identically distributed data across clients can lead to unfair and inefficient model performance. In this work, we introduce \textit{DRDM}, a novel algorithm that addresses these issues by combining a distributionally robust optimization (DRO) framework with dynamic regularization to mitigate client drift. \textit{DRDM} frames the training as a min-max optimization problem aimed at maximizing performance for the worst-case client, thereby promoting robustness and fairness. This robust objective is optimized through an algorithm leveraging dynamic regularization and efficient local updates, which significantly reduces the required number of communication rounds. Moreover, we provide a theoretical convergence analysis for convex smooth objectives under partial participation. Extensive experiments on three benchmark datasets, covering various model architectures and data heterogeneity levels, demonstrate that \textit{DRDM} significantly improves worst-case test accuracy while requiring fewer communication rounds than existing state-of-the-art baselines. Furthermore, we analyze the impact of signal-to-noise ratio (SNR) and bandwidth on the energy consumption of participating clients, demonstrating that the number of local update steps can be adaptively selected to achieve a target worst-case test accuracy with minimal total energy cost across diverse communication environments. 
\end{abstract}

\maketitle

\section{INTRODUCTION}
Federated learning (FL) has emerged as a promising solution for large-scale machine learning (ML), where training data is distributed across different clients \cite{9464278, 10.1145/3298981, 9599369}. FL enables model training on decentralized data by having clients share model updates (e.g. gradients or weights) with a central server, instead of directly sharing their raw data. This approach improves data privacy and can reduce communication overhead and latency compared to centralizing raw data. However, the practical implementation of FL faces several key challenges \cite{9084352}. One such challenge is partial participation, where the number of clients involved in training can fluctuate because unreliable communication links (e.g. deep fading or slow connections) may hinder their ability to connect to the server. Another significant issue is the high communication overhead \cite{10118646}. Although FL can reduce communication costs compared to transmitting raw data, the periodic transmission of modern ML models, often containing millions or billions of parameters, can still impose substantial overhead, which is particularly problematic for systems with limited bandwidth, such as wireless portable devices\cite{9685045}. Finally, data heterogeneity presents a major hurdle, as training data is frequently non-independent and identically distributed (non-IID) with variations in data imbalance and availability across clients. This heterogeneity is present in many practical use cases\cite{9090366, 9578660, 10.1007/978-3-030-58607-2_5
}. For example, disease samples drawn from different geographical locations exhibit regional variations in prevalence and characteristics, while handwriting data shows stylistic differences even for identical words\cite{fl_non_iid_survey, opt_fl_non_iid}. Such disparities can significantly affect the convergence and performance of the global ML model \cite{fl_non_iid_survey, opt_fl_non_iid}.

\textit{FedAvg} is a foundational algorithm that solves the classical empirical risk minimization (ERM) problem in FL \cite{localSGD, fl ,fedProx,
li2019practicalfederatedgradientboosting,
Wang2020Federated,
a96536,
10.1145/3510540,
he2020fedmlresearchlibrarybenchmark,
 fl_comm-eff}. Building on \textit{local SGD} \cite{localSGD}, clients perform multiple gradient descent steps on their local data before communicating with the server, thereby enhancing communication efficiency and accommodating partial client participation. However, \textit{FedAvg} often exhibits poor performance under data heterogeneity. This is largely attributed to client drift, where local updates diverge due to differing local datasets, causing bias to the global model \cite{Li2020On}. Although the resulting global model might achieve a reasonable average performance, its accuracy often degrades on individual clients' heterogeneous data \cite{fedProx}.

To mitigate these issues, several variants of \textit{FedAvg} have been proposed \cite{fedProx, FedDYN, scaffold, fedNova, scallion}. In particular, \textit{FedProx}  \cite{fedProx} incorporates a proximal term that regularizes local updates, while \textit{FedDyn} \cite{FedDYN} proposes local dynamic regularization designed to align local and global optima. Other methods, such as \textit{SCAFFOLD} \cite{scaffold} and \textit{SCALLION} \cite{scallion}, introduce local control variates to minimize local update variance and apply compression to increase communication efficiency, respectively. Despite these improvements, these methods fundamentally rely on minimizing the average of local empirical risks (ERM). Consequently, their performance can still degrade under significant data heterogeneity. This degradation occurs because the ERM objective implicitly assumes that local data distributions are similar to the global average distribution, an assumption that is often violated in practice. This issue becomes more pronounced as the level of heterogeneity of local data distributions increases, which highlights the limitations of ERM-type approaches in distributed settings.  

To address the fundamental limitations of ERM-type approaches in heterogeneous settings, recent work has shifted focus toward distributionally robust optimization (DRO), which aims to improve model robustness against distribution shifts in local data\cite{Li2020DittoFA, afl, FedRobust}. Specifically, DRO seeks solutions that perform well even under the worst-case distribution within a defined uncertainty set, often related to mixtures of client distributions in FL, thereby improving fairness by ensuring minimum performance levels across clients. This approach promotes a better generalization of the global model across heterogeneous client data. For example, \textit{DRFA} \cite{drfa} proposes a robust federated averaging method that targets uniform performance between clients by optimizing mixtures of local distributions. Specifically, \textit{DRFA} trains the model parameters in a distributed manner by performing multiple SGD steps on adaptively sampled clients, based on the dual variable distribution, to perform model averaging. While \textit{DRFA} addresses distributional robustness, it can still suffer from suboptimal solutions due to the client drift inherent in performing multiple local steps on heterogeneous data, similar to \textit{FedAvg}.
\subsection{RELATED WORKS}
\subsubsection{FEDERATED AVERAGING AND ITS EXTENSIONS}
\textit{FedAvg} \cite{fl} is a widely used optimization method in FL that extends \textit{local SGD} by allowing clients to perform multiple local updates before communicating with a central server. While \textit{FedAvg} is communication-efficient and performs well under IID conditions, its performance degrades significantly in non-IID settings due to the mismatch between the objectives of the local clients and the global objective \cite{Li2020On}. To address this, several variants of \textit{FedAvg} have been proposed. \textit{FedProx} \cite{fedProx} adds a proximal term to the local loss function to stabilize updates and limit client drift. However, this approach does not fully resolve the misalignment between local and global optima and relies on careful hyperparameter tuning. \textit{FedDyn} \cite{FedDYN} addresses this problem more directly by introducing a dynamic regularizer that adjusts the local objectives so that their stationary points align better with those of the global loss. \textit{SCAFFOLD} \cite{scaffold} introduces control variates to correct local update drift, while \textit{SCALLION} \cite{scallion} incorporates communication compression techniques to reduce bandwidth requirements. \textit{FedNova} \cite{fedNova} applies normalization strategies to account for heterogeneous local update steps. Despite their improvements, these methods remain fundamentally rooted in ERM and, therefore, are still sensitive to data heterogeneity.
\subsubsection{DISTRIBUTIONALLY ROBUST OPTIMIZATION IN FL}
Recent works have explored distributionally robust optimization (DRO) as a means to mitigate the limitations of ERM in federated settings. DRO approaches aim to learn models that perform well on the worst-case mixture of client distributions, thus improving robustness to data heterogeneity. Agnostic Federated Learning (\textit{AFL}) \cite{afl} frames the FL objective as a min-max optimization problem for all convex combinations of client distributions. \textit{FedRobust} \cite{FedRobust} extends this idea by proposing an algorithm that is robust against affine distribution shifts to the
distribution of observed samples. \textit{DRFA} \cite{drfa} builds on \textit{FedAvg} by incorporating a DRO objective that minimizes robust empirical loss over a mixture of local distributions, while still relying on periodic model averaging. Other approaches include \textit{DR-DSGD} \cite{DR_DSGD}, a Kullback–Leibler-regularized distributionally robust optimization method for decentralized learning, with the aim of improving worst-case accuracy, communication efficiency, and fairness across clients.
\subsubsection{CLIENT SAMPLING AND FAIRNESS IN FL}
Partial client participation in FL arises from constraints such as unreliable communication, device availability, and energy limitations. However, naive uniform sampling of clients can introduce bias and instability in the global model, especially in the presence of data heterogeneity. To mitigate this, several works have proposed intelligent client sampling and fairness-aware optimization strategies. \textit{q-FFL} \cite{Li2020Fair} introduces a fairness-adjusted objective that reweights client updates based on local performance, thereby promoting equitable accuracy across clients. \textit{AFL} \cite{afl}, which also appears in the context of distributional robustness, explicitly optimizes for the worst-case mixture over client distributions to ensure the model performs well on all clients. \textit{Clustered FL} \cite{clustered_fl} attempts to group clients with similar data distributions and perform aggregation within clusters to improve convergence in non-IID settings.
\subsection{CONTRIBUTIONS}
In this paper, we build on the DRO framework and propose to correct the client drift in the update of the model parameters by using the dynamic regularization technique introduced in \textit{FedDyn}\cite{FedDYN}. Specifically, we modify each client's local objective by adding a dynamic penalty term, which is updated at every round. This penalty is designed to align the stationary points of the modified local objectives with those of the global robust objective, thereby mitigating client drift within the DRO framework and improving convergence towards the robust optimum. We summarize our results and contributions below.
\begin{itemize}
    \item We propose \textit{DRDM}, a \textit{d}istributionally \textit{r}obust federated learning algorithm with client-\textit{d}rift \textit{m}inimization, where i) clients locally update their local model parameters with multiple steps using modified local objectives to mitigate the client-drift, based on \cite{FedDYN}, and ii) the global dual variable is updated periodically at the server side using a \textit{randomized snapshotting scheme} to approximate the accumulation of local models from the clients, as done in \cite{drfa}.
    \item We theoretically analyze the convergence of \textit{DRDM} for convex objectives. Our analysis covers the challenges of local SGD steps, drift-correction via dynamic regularization, partial participation, and periodic dual updates using randomized snapshotting. We prove convergence by bounding the duality gap and derive a convergence rate of $\mathcal{O}(1/T^{3/8})$, where $T$ represents the total number of local iterations.
    \item Through experiments on three benchmark datasets under varying non-IID settings, we show that \textit{DRDM} achieves superior communication efficiency compared to baselines, particularly improving the client's performance in terms of the worst-case test accuracy.
    \item We investigate the impact of signal-to-noise ratio (SNR) and bandwidth on the energy consumption of participating clients. Our results highlight that \textit{DRDM} can be tuned to achieve a target worst-case test accuracy with minimal total energy costs, by adaptively selecting the optimal number of local update steps for different communication regimes. 
\end{itemize}
The remainder of this paper is organized as follows. Section~\ref{probelm_formulation} presents the problem formulation and the proposed algorithm. Section~\ref{conv_analysis} provides the convergence analysis. Simulation results are discussed in Section~\ref{experiments}, and concluding remarks are presented in Section~\ref{conclusion}. The proof of convergence is provided in Appendices~\ref{overview_proof}--\ref{proof: theorem1}, while additional experimental results are presented in Appendix~\ref{extra_experiments}.
\section{PROBLEM FORMULATION \& PROPOSED ALGORITHM}\label{probelm_formulation}
We consider an FL setting consisting of $N$ clients collaboratively training a shared model under the coordination of a parameter server. Each client $i$ has a local dataset containing $d_i$ training samples drawn from an unknown distribution $\mathcal{P}_i$, which may differ between clients due to data heterogeneity. FL optimization problems are commonly cast as minimization of a weighted sum of local loss functions, which is typically solved in a distributed manner by minimizing the following global objective
    \begin{align} \label{FL_problem}
	\min_{{\boldsymbol{w}} \in \mathcal{W}}\,\,  f(\boldsymbol{w}) \triangleq  \displaystyle\sum_{i=1}^{N}\frac{d_i}{d} f_i(\boldsymbol{w}),
\end{align} 
where $d = \sum_{i=1}^N d_i$ is the total number of samples, $f_i(\boldsymbol{w})= \mathbb{E}_{\zeta \sim \mathcal{P}_i}[\ell_i(\boldsymbol{w}; \zeta)]$ is the local objective of client $i$ evaluated at the model parameters $\boldsymbol{w}$ for a given loss function $\ell_i$, and $\mathcal{W}$ is a closed convex set of feasible model parameters. 

Although the ERM objective in \eqref{FL_problem} performs well with homogeneous data distributions, it fundamentally optimizes for average performance rather than fair outcomes. When $\mathcal{P}_i \ne \mathcal{P}_j$ for some $i \ne j$, as is common in real-world federated settings, this approach can lead to suboptimal or unfair performance. Specifically, the average-based objective may favor clients with larger datasets or less challenging data distributions, potentially sacrificing performance on underrepresented clients.

To address this issue, we adopt the agnostic FL framework introduced in \cite{afl}, which formulates the global training objective as a distributionally robust optimization (DRO) problem. Instead of minimizing the average loss, we aim to find a model that performs well across all possible mixtures of client distributions
       \begin{align}\label{dro}
	\min_{{\boldsymbol{w}}   \in \mathcal{W}}\max_{{\boldsymbol{\lambda}}   \in \Lambda}\,\,  F(\boldsymbol{w}, \boldsymbol{\lambda})  \triangleq   \displaystyle\sum_{i=1}^{N} \lambda_i f_i(\boldsymbol{w}),
\end{align}  
where $\boldsymbol{\lambda} \in \Lambda \, \triangleq \, \{  \boldsymbol{\lambda} \in \mathbb{R}^N_+: \displaystyle\sum_{i=1}^{N}\lambda_i = 1\}$ is the dual variable with values corresponding to each local loss function. This min-max formulation explicitly targets robustness by optimizing for the worst-performing client distributions. Unlike ERM, which may achieve good average performance while allowing poor outcomes on challenging distributions, the DRO approach ensures a minimum performance guarantee across all clients, making it particularly well-suited for heterogeneous federated environments. Our goal is to solve the DRO problem in (\ref{dro}) efficiently and collaboratively, using limited communication between the clients and the server.

Our proposed method, \textit{DRDM}, outlined in Algorithm~\ref{alg:DRDM}, addresses the challenges of data heterogeneity and client drift through two main components: i) a client-side local model update mechanism and ii) a server-side periodic dual variable update strategy. The local model update at the client side is based on the dynamic regularization technique from \textit{FedDyn} \cite{FedDYN}. The dual variable $\boldsymbol{\lambda}$ is updated and used to dynamically draw the participating clients at every communication round $s \in [S]$.
We let $\tau$ be the number of local update steps that are executed at the client side between every consecutive communication rounds.
\subsubsection{LOCAL MODEL UPDATE} 
Let $\mathcal{N}$ denote the set of all clients. At communication round $(s+1)$, a subset of clients $\mathcal{D}^{(s)} \subset \mathcal{N}$ is randomly selected by the server with probabilities proportional to the elements of the dual variable $\boldsymbol{\lambda}^{(s)}$, then the server broadcasts the global primal model $\boldsymbol{\bar{w}}^{(s)}$ of the previous communication round $s$ to the selected clients $i \in \mathcal{D}^{(s)}$. Each client $i$, after receiving the model parameters, aims to approximately minimize a modified local objective function, which is a combination of the local loss function and a penalty risk term, where the latter is based on the current client's local model and the received global model as follows
      \begin{align}\label{modified_emperical_risk_objective}
      \boldsymbol{w}_i^{((s+1)\tau)} =& \,\,\underset{  
\boldsymbol{w}}{\arg\min} \,\,\,  \Big [\mathcal{F}(\boldsymbol{w}, \boldsymbol{w}_i^{(s\tau)}, \boldsymbol{\bar{w}}^{s}) \triangleq  f_i(\boldsymbol{w}) \nonumber \\& - \langle {\nabla} f_i(\boldsymbol{w}_i^{(s\tau)}), \boldsymbol{w}\rangle + \frac{\mu}{2} \lVert \boldsymbol{w} - \boldsymbol{\bar{w}}^{(s)} \rVert^2 \Big].
      \end{align}
The local gradient term ${\nabla} f_i(\boldsymbol{w}_i^{(s\tau)})$\footnote{${\nabla} f_i(\boldsymbol{w}_i^{(s\tau)})$ is computed once per round $s$ and remains fixed across all local updates $t \in [s\tau, (s+1)\tau)$.} is updated recursively while satisfying the first-order condition for local optima 
\begin{align}\label{local_optima_first_order_condition}
      {\nabla} f_i(\boldsymbol{w}_i^{((s+1)\tau)}) - {\nabla} f_i(\boldsymbol{w}_i^{(s\tau)}) + \mu  (\boldsymbol{w}_i^{((s+1)\tau)} - \boldsymbol{\bar{w}}^{(s)}) = \boldsymbol{0}.
      \end{align}
The solution to the optimization problem in (\ref{modified_emperical_risk_objective}) is approximated by performing $\tau$ update steps at each client side (see Algorithm \ref{alg:local_steps}). In particular, let $\boldsymbol{w}_i^{(t+1)}$ be the local model at client $i$ at local step $(t+1)$ within the round $(s+1)$. At each local iteration $t = s\tau, \dots, (s+1)\tau-1$, the client $i \in \mathcal{D}^{s}$ updates $\boldsymbol{w}_i^{(t+1)}$ as follows  
\begin{align}
\boldsymbol{w}_i^{(t+1)} = \prod_{\mathcal{W}}(\boldsymbol{w}_i^{(t)} - \eta \boldsymbol{d}_i^{(t)}),
\end{align}
where $\prod_{\mathcal{W}}(\cdot)$ denotes the projection onto the set $\mathcal{W}$, $\eta$ is the learning rate, and $\boldsymbol{d}_i^{(t)}$ is the modified gradient, computed as
\begin{align}
    \boldsymbol{d}_i^{(t)} = {\nabla}{\ell_i(\boldsymbol{w}^{(t)}_i; \boldsymbol{\zeta}_i^{(t)})}-  {\nabla}f_i(\boldsymbol{w}_i^{(s\tau)}) - \mu (\boldsymbol{\bar{w}}^{(s)} - \boldsymbol{w}_{i}^{(t)}),
\end{align}
where $\boldsymbol{\zeta}_i^{(t)}$ is the random data sample from the local dataset of client $i$. After $\tau$ steps, each client $i \in \mathcal{D}^{(s)}$ sends its local model $\boldsymbol{w}_i^{((s+1)\tau)}$ to the server to update the global model $\boldsymbol{\bar{w}}^{(s+1)}$, as seen in step 20 of Algorithm \ref{alg:DRDM}. Stale clients do not update their local models and do not participate in the update of the global model. This process is done at every communication round $s$. 
\begin{algorithm}[t]
\caption{\strut Distributionally Robust Federated Learning with Client-drift
Minimization (DRDM)}\label{alg:DRDM}
\begin{algorithmic}[1] 
\STATE \textbf{Input}: $N, T, \tau, S = T/\tau, \eta, \gamma, m, \mu,  f_i(\cdot), ~\forall i \in \mathcal{N}$ 
\STATE \textbf{Output}: $\boldsymbol{w}^{(S)}$, $\boldsymbol{\lambda}^{(S)}$
\STATE \textbf{Initialization}: $ \boldsymbol{w}^{(0)},  {\lambda}^{(0)}_i, \boldsymbol{c^{(0)}}\,\,\,\,\, \forall i$.
\FOR {$s =0$ to $S-1$}
\STATE Server \textbf{samples} clients $\mathcal{D}^{(s)}  \subset \mathcal{N}$ according to $\boldsymbol{\lambda}^{(s)}$ with size $m$.
\STATE Server \textbf{samples} $t^{'}$ from $s\tau+1, \dots, (s+1)\tau$ uniformly at random. 
\STATE Server \textbf{broadcasts} $\boldsymbol{\bar{w}}^{(s)}$ to all clients $i \in \mathcal{D}^{(s)}$.
\\
\FOR {For each client $i \in \mathcal{D}^{(s)}$ in parallel}   
\STATE $ \boldsymbol{w}_i^{((s+1)\tau)},$ $\boldsymbol{w}_i^{(t^{'})}  \leftarrow $ \textit{LocalUpdateSteps}$(\boldsymbol{\bar{w}}^{(s)}, {\nabla}f_i(\boldsymbol{w}_i^{(s\tau)}), \mu, \eta, \tau, t^{'} )$\,\, 
\STATE ${\nabla}f_i(\boldsymbol{w}_i^{((s+1)\tau)}) \leftarrow {\nabla}f_i(\boldsymbol{w}_i^{(s\tau)})  - \mu \,(\boldsymbol{w}_i^{((s+1)\tau)} - \, \boldsymbol{\bar{w}}^{(s)})$ \,\,\,\,\,\,\,\,\,\,\ 
\STATE Transmit $\boldsymbol{w}_i^{((s+1)\tau)}$ and $\boldsymbol{w}_i^{(t^{'})}$ 
\ENDFOR
\\
\FOR {For each client $i \notin \mathcal{D}^{(s)}$ in parallel}  
\STATE $ \boldsymbol{w}_i^{((s+1)\tau)} = \boldsymbol{w}_i^{(s\tau)}$ and $ {\nabla}f_i(\boldsymbol{w}_i^{((s+1)\tau)}) = {\nabla}f_i(\boldsymbol{w}_i^{(s\tau)})$
\ENDFOR
\\
\STATE Server \textbf{Computes}:\,\,\,
\STATE $\boldsymbol{c}^{(t^{'})} \leftarrow \boldsymbol{c}^{(s)} -  \frac{\mu}{N} (  \sum_{i \in \mathcal{D}^{(s)}} \boldsymbol{w}_i^{(t^{'})} \, - \, \boldsymbol{\bar{w}}^{(s)})$ 
\STATE $\boldsymbol{c}^{(s+1)} \leftarrow \boldsymbol{c}^{(s)} -  \frac{\mu}{N} (  \sum_{i \in \mathcal{D}^{(s)}} \boldsymbol{w}_i^{((s+1)\tau)} \, - \, \boldsymbol{\bar{w}}^{(s)})$ 
\STATE $\boldsymbol{{w}}^{(t^{'})} \leftarrow   \,\, \big(\frac{1}{m}\sum_{i \in \mathcal{D}^{(s)}} \boldsymbol{w}_i^{(t^{'})}\big)\,\, {-\,\, \frac{1}{\mu} \boldsymbol{c}^{(t^{'})}}$
\STATE $\boldsymbol{\bar{w}}^{(s+1)} \leftarrow   \,\, \big(\frac{1}{m}\sum_{i \in \mathcal{D}^{(s)}} \boldsymbol{w}_i^{((s+1)\tau)}\big)\,\, {-\,\, \frac{1}{\mu} {\boldsymbol{c}^{(s+1)}}}$  
\STATE $\boldsymbol{\lambda}^{(s+1)} = $ \textit{DualVariableUpdate}$(\gamma, \tau, {\boldsymbol{w}}^{(t^{'})}, \boldsymbol{\lambda}^{(s)})$ 
\ENDFOR
\end{algorithmic}
\end{algorithm}
\begin{algorithm}[ht]
\caption{\strut LocalUpdateSteps}\label{alg:local_steps}

\begin{algorithmic}[1] 
\STATE \textbf{Input}:$ \boldsymbol{\bar{w}}^{(s)}, {\nabla}f_i(\boldsymbol{w}_i^{(s\tau)}), \mu, \eta, \tau, t^{'}, \mathcal{W} $
\\
\STATE \textbf{Output}: $ \boldsymbol{w}_i^{((s+1)\tau)}, \, \boldsymbol{w}_i^{(t^{'})}$
\\
\STATE $\boldsymbol{w}^{(s\tau)}_i \leftarrow \boldsymbol{\bar{w}}^{(s)}$ 
\\
\FOR {$t =s\tau, \dots,(s+1)\tau -1$} 
\STATE  $\boldsymbol{d}_i^{(t)} = {\nabla}{\ell_i(\boldsymbol{w}^{(t)}_i; \boldsymbol{\zeta}_i^{(t)})}-  {\nabla}f_i(\boldsymbol{w}_i^{(s\tau)}) - \mu (\boldsymbol{\bar{w}}^{(s)} - \boldsymbol{w}_{i}^{(t)})$ 
\STATE $\boldsymbol{w}_i^{(t+1)} = \prod_{\mathcal{W}}( \boldsymbol{w}_i^{(t)} - \eta \boldsymbol{d}_i^{(t)})$ 
\ENDFOR
\end{algorithmic}
\end{algorithm}

\begin{algorithm}[ht]
\caption{\strut  DualVariableUpdate  }\label{alg:dualVariableUpdate}
\begin{algorithmic}[1] 
\STATE \textbf{Input}: $\gamma, \tau, {\boldsymbol{w}}^{(t^{'})}, \boldsymbol{\lambda}^{(s)}$  
\\
\STATE \textbf{Output}:  $\boldsymbol{\lambda}^{(s+1)}$
\\
\STATE  Server uniformly samples a subset $\mathcal{U} \subset \mathcal{N}$ of clients with size $m$ and broadcasts ${\boldsymbol{w}}^{(t^{'})}$ to $i \in \mathcal{U}$ 
\\
\FOR {For each client $i \in \mathcal{U}$ in parallel}
\STATE  Compute $\ell_i({\boldsymbol{w}}^{(t^{'})}; \boldsymbol{\zeta}_i)$ and transmit it to the server. 
\ENDFOR
\\
\STATE Server constructs the vector $\boldsymbol{v}: v_i = \frac{N}{m} \ell_i({\boldsymbol{w}}^{(t^{'})}; \boldsymbol{\zeta}_i)$ if $i \in \mathcal{U}$, otherwise $v_i = 0$
\STATE $\boldsymbol{\lambda}^{(s+1)} =  \Pi_{\boldsymbol{\Lambda}} \Big(  \boldsymbol{\lambda}^{(s)} + \tau \gamma \boldsymbol{v}  \Big)$ 
\end{algorithmic}
\label{updateMixParam}
\end{algorithm}
\subsubsection{DUAL VARIABLE UPDATE}
The dual variable $\boldsymbol{\lambda}$ controls the mixture of different
local losses and can only be updated by the server at each communication round. In the absence of regularization on $\boldsymbol{\lambda}$, we have a linear problem with respect to $\boldsymbol{\lambda}$. Hence, the gradient of $\boldsymbol{\lambda}$ only depends on $\boldsymbol{w}$, so the approximation of the sum of the history gradients over the previous local loop can be done. In the fully synchronized setting, i.e., the clients communicate with the server at every iteration $t$, the dual variable $\boldsymbol{\lambda}$ can be updated as follows
\begin{align} \label{dualVariableUpdate_step}
    \boldsymbol{\lambda}^{(s+1)} = \prod_{\Lambda} \Big(\boldsymbol{\lambda}^{(s)} + \gamma       \displaystyle\sum_{t=s\tau +1}^{(s+1)\tau} {\nabla}_{\boldsymbol{\lambda}} F(\boldsymbol{w}^{(t)}, \boldsymbol{\lambda}^{(s)})\Big),
\end{align}
where $\prod_{\Lambda}(\cdot)$ is the projection onto ${\Lambda}$ and $\boldsymbol{w}^{(t)} = (\frac{1}{m} \sum_{i \in \mathcal{D}^{(s)}} \boldsymbol{w}_i^{(t)}) - \frac{1}{\mu}\boldsymbol{c}^t$ is the global model at local iteration $t$. Since we have a limited number of communication rounds, the update in (\ref{dualVariableUpdate_step}) is approximated using the random snapshotting scheme introduced in \cite{drfa}. This approximation is implemented by the function $\textit{DualVariableUpdate}$ (Algorithm \ref{alg:dualVariableUpdate}), which is called in line 21 of Algorithm \ref{alg:DRDM}. First, in the communication round $(s+1)$, the server samples a random iteration $t^{'}$ from $s\tau +1$ to $(s+1)\tau$ and broadcasts it to the participating clients $\mathcal{D}^{(s)}$. After the update of the local model is performed, the clients send their local model $\boldsymbol{w}_i^{(t^{'})}$ sampled at step $t^{'}$ to the server. The server then computes the global model snapshot $\boldsymbol{w}^{(t^{'})}$ based on received local models, as outlined in line 19 of Algorithm~\ref{alg:DRDM}. Next, the server broadcasts the model $\boldsymbol{w}^{(t^{'})}$ to the a set $\mathcal{U}$ of $m$ clients chosen uniformly at random. Each selected client evaluates their loss function $\ell_i(.)$ at $\boldsymbol{w}^{(t^{'})}$ using a random local data mini-batch $\boldsymbol{\zeta}_i$, and sends it to the server. The latter uses the received loss functions and constructs the vector $\boldsymbol{v}$, as seen in line 7 of Algorithm \ref{alg:dualVariableUpdate}. Finally, the dual variable is updated as shown in line 8 of Algorithm \ref{alg:dualVariableUpdate}. The computed stochastic gradient is an $\textit{unbiased estimate}$ since it satisfies the following
\begin{align} \label{unbiased_estimate}
    \mathbb{E}_{t^{'}, \mathcal{U}, \zeta_i}[\tau \boldsymbol{v}] &= \mathbb{E}_{t^{'}} \Big[ 
    \tau {\nabla}_{\boldsymbol{\lambda}} F (\boldsymbol{w}^{(t^{'})}, \boldsymbol{\lambda}^{(s)}) 
    \Big] \nonumber \\  &= \displaystyle\sum_{t=s\tau+1}^{(s+1)\tau} {\nabla}_{\boldsymbol{\lambda}} F(\boldsymbol{w}^{t}, \boldsymbol{\lambda}^{s}).
\end{align}
\section{CONVERGENCE ANALYSIS}\label{conv_analysis}
In this section, we present the theoretical guarantees of the \textit{DRDM} algorithm for a general class of convex smooth loss functions. We begin by analyzing the convergence behavior of Algorithm~\ref{alg:DRDM}, which optimizes the distributionally robust optimization problem in (\ref{dro}). The primary objective is to study the primal-dual gap, which reflects the distance to a saddle point of the min-max objective. Our algorithm incorporates periodic updates of the dual variable, decoupled from the primal model updates. Additionally, the introduction of a penalty term in local updates, along with gradient stochasticity, adds complexity to the convergence analysis. The main challenge lies in bounding the deviation between local and virtual iterates and ensuring that these deviations do not significantly impact convergence.
\begin{definition}[Weighted Gradient Dissimilarity] 
A set of local objectives $f_i(\cdot), \,i \in \mathcal{N}$  exhibit $
\Gamma$ gradient dissimilarity defined as
$    \Gamma := \sup_{\bm{w} \in \mathcal{W}, \boldsymbol{p}\in \Lambda,i\in \mathcal{N},}  \sum_{j\in \mathcal{N}}p_j\|\nabla f_i(\bm{w})- \nabla f_{j}(\bm{w})\|^2$.
\end{definition}
The above definition captures how much the gradients of different clients' objectives can diverge from each other, when weighted by the distribution $\boldsymbol{p}$. It generalizes standard notions of gradient dissimilarity previously used in the context of \textit{local SGD} for federated optimization~\cite{li2021, Li2020On}. When all local objectives are identical, the gradient dissimilarity measure becomes zero. The appearance of $\Gamma$ in our convergence limits reflects the impact of data heterogeneity. As clients compute updates from their local data, statistical differences between local data distributions can slow convergence and introduce bias into the global primal model.

To derive meaningful bounds, we adopt the following standard assumptions.
\begin{assumption} [Smoothness/Gradient Lipschitz]\label{assumption: smoothness}
 Each component function $f_i(\cdot), \,i \in \mathcal{N}$ and global function $F(\cdot, \cdot)$  are $L$-smooth, which implies: $
    \| \nabla f_i(\bm{x}_1) - \nabla f_i(\bm{x}_2)\| \leq  L \|\bm{x}_1 - \bm{x}_2\|, \forall i \in \mathcal{N}, \forall \bm{x}_1, \bm{x}_2$ and $ 
    \| \nabla F(\bm{x}_1,\bm{y}_1) - \nabla F(\bm{x}_2,\bm{y}_2)\| \leq  L \|(\bm{x}_1,\bm{y}_1) - (\bm{x}_2,\bm{y}_2)\|, \forall (\bm{x}_1,\bm{y}_1), (\bm{x}_2,\bm{y}_2)$.
\end{assumption}
 
\begin{assumption} [Gradient Boundedness] \label{assumption: bounded gradient} The gradient w.r.t $\boldsymbol{w}$ and $\boldsymbol{\lambda}$ are  bounded, i.e., $\mathbb{E}\|\nabla f_i(\boldsymbol{w})\|\leq G_w $ and  $\mathbb{E}\|\nabla_{\boldsymbol{\lambda}}F(\boldsymbol{w},\boldsymbol{\lambda})\|\leq G_{\lambda}$.
 \end{assumption}

\begin{assumption} [Bounded Domain] \label{assumption: bounded domain} The diameters of $\mathcal{W}$ and $ \Lambda$ are bounded by $D_{\mathcal{W}}$ and $D_{\Lambda}$.
\end{assumption}

\begin{assumption}[Bounded Variance] \label{assumption: bounded variance}  Let $\Tilde{\nabla} F(\bm{w};\boldsymbol{\lambda}) $ be the stochastic gradient for $\boldsymbol{\lambda}$, which is the $N$-dimensional vector such that the $i$th entry is $\ell_i(\bm{w};\xi)$, and the rest are zero. Then we assume $
   \mathbb{E} \|\nabla \ell_i(\bm{w};\xi) - \nabla f_i(\bm{w})\|^2 \leq  \sigma^2_{w}, \forall i \in \mathcal{N}$ and $
    \mathbb{E} \|\Tilde{\nabla} F(\bm{w};\boldsymbol{\lambda}) - \nabla F(\bm{w};\boldsymbol{\lambda})\|^2 \leq  \sigma^2_{\lambda}$.
\end{assumption}
Before we establish the lemmas that are used to prove the theorem, let us introduce some useful variables for ease of analysis. We define virtual sequences $\{ \bm{w}^{(t)}\}_{t=1}^T$ and $\{ \bm{c}^{(t)}\}_{t=1}^T$ that will be used in our proof, and some intermediate variables
\begin{align*} 
       \bm{w}^{(t)} &= \frac{1}{m}\sum_{i\in{\mathcal{D}^{(\lfloor{\frac{t}{\tau}}\rfloor)}}} \bm{w}^{(t)}_i - \frac{1}{\mu}\boldsymbol{c}^{t},  \\ \bm{c}^{(t)} &= \bm{c}^{(t-1)} - \frac{\mu}{N} \sum_{i\in{\mathcal{D}^{(\lfloor{\frac{t}{\tau}}\rfloor)}}}(\bm{w}_i^t - \bm{w}^{t-1}),  \\ \Bar{\bm{u}}^{(t)} &= \frac{1}{m}\sum_{i\in \mathcal{D}^{(\lfloor{\frac{t}{\tau}}\rfloor)}}\Big(  \nabla f_i(\bm{w} ^{(t)}_i)\Big),   \\
        \bm{u}^{(t)} &= \frac{1}{m}\sum_{i\in{\mathcal{D}^{(\lfloor{\frac{t}{\tau}}\rfloor)}}} \Big( \nabla \ell_i( \bm{w} ^{(t)}_i;\xi^{(t)}_i) \Big),     \\
        \Bar{\bm{v}}^{(t)} &= \nabla_{\boldsymbol{\lambda}} F(\bm{w} ^{(t)},\boldsymbol{\lambda}) = \left[f_1(\bm{w}^{(t)}), \ldots,f_N(\bm{w} ^{(t)})\right],   \\
         \bar{\Delta}_{s} &= \sum_{t=s\tau+1}^{(s+1)\tau} \gamma \Bar{\bm{v}}^{(t)}, \\
         \Delta_{s} &= \tau \gamma \bm{v}, \\
         \delta^{(t)} &= \frac{1}{m} \sum_{i\in\mathcal{D}^{(\lfloor{\frac{t}{\tau}}\rfloor)}} \left\|\bm{w}^{(t)}_i - \bm{w}^{(t)}\right\|^2, \\
          \beta^{(t)} &= \frac{1}{m}\sum_{i\in\mathcal{D}^{(\lfloor{\frac{t}{\tau}}\rfloor)}} \left\|(\boldsymbol{w}_{i}^{(t)} - \boldsymbol{{w}}^{({\lfloor{\frac{t}{\tau}}\rfloor}\tau)})\right\|^2, \\
         \varphi^{(t)} &= \frac{1}{m} \sum_{i\in{\mathcal{D}^{(\lfloor{\frac{t}{\tau}}\rfloor)}}} \| \bm{w}_i^{t+1} - \bm{w}^t \|^2, 
         \nonumber
\end{align*}
where $\bm{v} \in \mathbb{R}^N$ is the stochastic gradient of the dual variable, generated by Algorithm~\ref{alg:dualVariableUpdate} to update $\boldsymbol{\lambda}$, such that ${v}_i = \ell_i(\bm{w}^{(t')};\xi_i)$ for $i \in \mathcal{U} \subset \mathcal{N}$ where $\xi_i$ is stochastic minibatch sampled from $i$th local data shard,  and $t'$ is the snapshot index  sampled from $s\tau+1$ to $(s+1)\tau$.
\begin{lemma}
\label{lemma: bounded variance of w}
The stochastic gradient $\bm{u}^{(t)}$ is unbiased, and its variance is bounded, which implies
\begin{align}
\mathbb{E}_{\xi_i^{(t)},\mathcal{D}^{(\lfloor{\frac{t}{\tau}}\rfloor)}}\left[\bm{u}^{(t)}\right] &=  \mathbb{E}_{\mathcal{D}^{(\lfloor{\frac{t}{\tau}}\rfloor)}}\left[\Bar{\bm{u}}^{(t)}\right]  \nonumber  \\ &= \mathbb{E} \left[   \sum_{i=1}^N \lambda^{(\lfloor{\frac{t}{\tau}}\rfloor)}_i \nabla f_i(\bm{w}^{(t)}_i) \right],\nonumber  \\ \mathbb{E}\left[\|\bm{u}^{(t)} - \Bar{\bm{u}}^{(t)}\|^2\right] &=  \frac{\sigma^2_{w}}{m}. 
\end{align}
\begin{proof}
The unbiasedness is due to the sampling of the clients based on $\boldsymbol{\lambda}^{(\lfloor{\frac{t}{\tau}}\rfloor)}$. The variance is computed using the identity $\mathrm{Var}(\sum_{i=1}^m \bm{X}_i) = \sum_{i=1}^m \mathrm{Var}(\bm{X}_i)$.
\end{proof}
\end{lemma}

\begin{lemma}
\label{lemma: bounded variance of lambda}
The stochastic gradient at $\boldsymbol{\lambda}$ generated by Algorithm~\ref{alg:dualVariableUpdate} is unbiased, and its variance is bounded
\begin{equation}
    \mathbb{E}[\Delta_{s}] =  \bar{\Delta}_{s}, \quad \quad \mathbb{E}[\|\Delta_{s} - \bar{\Delta}_{s}\|^2] \leq \gamma^2\tau^2\frac{\sigma_{\lambda}^2}{m}.\label{l10}
\end{equation} 
\begin{proof}
The unbiasedness is due to the fact that we sample the clients uniformly. The variance term is due to the identity $\mathrm{Var}(\sum_{i=1}^m \bm{X}_i) = \sum_{i=1}^m \mathrm{Var}(\bm{X}_i)$.
\end{proof}
\end{lemma}
\begin{lemma}
\label{lemma: deviation1}
The expected average squared norm distance of local models  $\bm{w} ^{(t)}_i, i \in \mathcal{D}^{(\lfloor{\frac{t}{\tau}}\rfloor)}$ and the (virtual) global model $\bm{w} ^{(t)}$ is bounded as follows:
\begin{align}
     \frac{1}{T}\sum_{t=0}^{T}\mathbb{E}\left[\delta^{(t)}\right]  
     &\leq 18\eta^2\tau^2 \Big( \sigma_w^2 + \frac{\sigma_w^2}{m} + \Gamma  \nonumber \\
     &   \quad + 2G_w^2 + 2\mu^2 D_{\mathcal{W}}^2 \Big),
\end{align}
where expectation is taken over sampling of clients at each iteration.
\end{lemma}
\begin{proof}
The details of the proof can be found in Appendix \ref{proof_deviation1}. 
\end{proof}
\begin{lemma}
\label{lemma: deviation2}
The expected average squared norm distance of local models  $\bm{w} ^{(t)}_i, i \in \mathcal{D}^{(\lfloor{\frac{t}{\tau}}\rfloor)}$ and the global model $\boldsymbol{{w}}^{({\lfloor{\frac{t}{\tau}}\rfloor}\tau)}$ is bounded as follows:
\begin{align}
     \frac{1}{T}\sum_{t=0}^{T}\mathbb{E}\Big[ \beta^{(t)} \Big] \leq  6\eta^2\tau^2  \Big(\sigma_w^2 + 4 G_w^2 \Big),
\end{align}
where expectation is taken over sampling of clients at each iteration.
\end{lemma}
\begin{proof}
This lemma, in conjunction with Lemma \ref{lemma: deviation1}, is used to derive the bound stated in Lemma~\ref{lemma: deviation3}. The complete proof is provided in Appendix~\ref{proof: deviation2}.
\end{proof}
\begin{lemma}
\label{lemma: deviation3}
The expected average squared norm distance of local models  $\bm{w} ^{t+1}_i, i \in \mathcal{D}^{(\lfloor{\frac{t}{\tau}}\rfloor)}$ and the (virtual) global model $\boldsymbol{w}^{t}$ is bounded as follows
\begin{align}
     &\frac{1}{T}\sum_{t=0}^{T}\mathbb{E}\Big[ \varphi^{(t)} \Big] \nonumber \\ & \quad \leq   72\eta^2\tau^2  \Big(\sigma_w^2+\frac{\sigma_w^2}{m} + \Gamma + 2G_w^2 + 2\mu^2 D_{\mathcal{W}}^2 \Big) \nonumber \\& \quad \quad +  4\eta^2 \frac{\sigma_w^2}{m} + 16 \eta^2G_w^2\nonumber \\& \quad \quad + 6\eta^4\tau^2  \mu^2\Big(\sigma_w^2 + 4 G_w^2 \Big),
\end{align}
where expectation is taken over sampling of clients at each iteration.
\end{lemma}
\begin{proof}
The details of the proof can be found in Appendix \ref{proof: deviation3}.
\end{proof}
\begin{lemma}
\label{lemma: one iteration w}
Under the assumptions of Theorem~\ref{theorem1}, we analyze a single iteration of the primal model $\bm{w} \in \mathcal{W}$ and establish the following result
{\begin{align}
&\mathbb{E}\|\bm{w}^{(t+1)} - \bm{w} \|^2 \nonumber \\ & \quad \leq \mathbb{E}\|\bm{w}^{(t)} - \bm{w} \|^2 \nonumber \\ &\quad \quad-2\eta\mathbb{E}\big[F(\bm{w} ^{(t)},\boldsymbol{\lambda}^{(\lfloor{\frac{t}{\tau}}\rfloor)})- F(\bm{w} ,\boldsymbol{\lambda}^{(\lfloor{\frac{t}{\tau}}\rfloor)})\big] \nonumber \\ &\quad \quad+ L\eta \mathbb{E}\Big[ \delta^{(t)} \Big] + 2\eta^2 G_w^2 + 4 \eta^2 \mathbb{E}\left[\|\bm{u}^{(t)} - \Bar{\bm{u}}^{(t)}\|^2\right]\nonumber \\ &\quad \quad + 4\eta^2 \mu^2 \mathbb{E} \Big[\beta^{(t)} \Big] + 4\mathbb{E} \Big[\varphi^{(t)} \Big].
\end{align}}
\end{lemma}
\begin{proof}
This lemma establishes a one-step progress bound for the primal iterates.  The details of the proof are deferred to Appendix \ref{proof: one iteration w}.   
\end{proof}
Using Lemma \ref{lemma: one iteration w}, we apply the telescoping sum from $t=1$ to $T$ and divide by $T$ to get the following bound
\begin{align}\label{sum_w}
    &\frac{1}{T}\sum_{t=1}^T \mathbb{E}(F(\bm{w} ^{(t)},\boldsymbol{\lambda}^{(\lfloor{\frac{t}{\tau}}\rfloor)})-F(\bm{w}  ,\boldsymbol{\lambda}^{(\lfloor{\frac{t}{\tau}}\rfloor)}))\nonumber\\
  & \quad \leq \frac{D_{\mathcal{W}}^2}{2T\eta} +  9 L\eta^2\tau^2  \left(\sigma_w^2+\frac{\sigma_w^2}{m} + \Gamma + 2G_w^2 + 2\mu^2 D_{\mathcal{W}}^2 \right) \nonumber\\
  & \quad\quad+ \eta G_w^2 + 10\frac{\eta \sigma^2_w}{m} + 24   \eta^3\tau^2  \mu^2  \Big(\sigma_w^2 + 4 G_w^2 \Big)+ 32 \eta G_w^2 \nonumber\\
  & \quad\quad+ 144\eta\tau^2  \left(\sigma_w^2+\frac{\sigma_w^2}{m} + \Gamma + 2G_w^2 + 2\mu^2 D_{\mathcal{W}}^2  \right).
\end{align}
\begin{lemma}
\label{lemma: one iteration lambda}
Next, we analyze one iteration of the dual update, the following holds for any $\boldsymbol{\lambda} \in \Lambda$
\begin{align}
 &\mathbb{E} \|\boldsymbol{\lambda}^{(s+1)} - \boldsymbol{\lambda} \|^2 \nonumber\\ &\quad \leq \mathbb{E}\|\boldsymbol{\lambda}^{(s)}- \boldsymbol{\lambda} \|^2 \nonumber\\ &\quad \quad-\sum_{t=s\tau+1}^{(s+1)\tau} \mathbb{E}[2\gamma(F(\bm{w} ^{(t)},\boldsymbol{\lambda})-F(\bm{w} ^{(t)},\boldsymbol{\lambda}^{(s)}))] \nonumber\\ &\quad \quad+ \mathbb{E}\|\bar{\Delta}_{s}\|^2 +  \mathbb{E}\|\Delta_{s} - \bar{\Delta}_{s}\|^2. 
\end{align}
\end{lemma}
\begin{proof}
Similarly to the primal case, this lemma establishes a one-step progress bound for the dual variable.  The proof can be found in Appendix \ref{proof: one iteration lambda}.     
\end{proof}
The telescoping sum from $S=0$ to $S-1$ and dividing by $T$ results in the following inequality  
\begin{align}\label{sum_lambda}
    &\frac{1}{T}\sum_{s=0}^{S-1} \sum_{t=s\tau+1}^{(s+1)\tau} \mathbb{E}(F(\bm{w} ^{(t)},\boldsymbol{\lambda})-F(\bm{w} ^{(t)},\boldsymbol{\lambda}^{(s)}))\nonumber\\
    &  \quad  \leq  \frac{1}{2\gamma T}\|\boldsymbol{\lambda}^{(0)}- \boldsymbol{\lambda}\|^2 + \frac{\gamma\tau }{2}G_{\lambda}^2+ \frac{\gamma\tau\sigma_{ \lambda}^2 }{2m} \nonumber\\
    & \quad \leq \frac{D_{\Lambda}^2}{2\gamma T}  + \frac{\gamma \tau G_{\boldsymbol{\lambda}}^2}{2}  + \frac{\gamma\tau\sigma_{ \lambda}^2 }{2m}.
\end{align}
By putting the inequalities in (\ref{sum_w}) and (\ref{sum_lambda}) together, and taking the maximum over the dual $\boldsymbol{\lambda}$ and the minimum over the primal $\bm{w}$, we present the main theorem that states the convergence of our algorithm.
\setcounter{assumption}{0}
\begin{theorem} \label{theorem1}
Let each local function $f_i$ be convex and the global function $F$ be linear in $\boldsymbol{\lambda}$. Assume that the conditions in Assumptions~\ref{assumption: smoothness}-\ref{assumption: bounded variance} hold. Solving (\ref{dro}) using Algorithm~\ref{alg:DRDM} with local steps $\tau =  \frac{T^{1/4}}{\sqrt{m}}$, learning rates $\eta = \frac{1}{4L \sqrt{T}}$, $\gamma = \frac{1}{T^{5/8}}$, and $\mu = 2L \sqrt{\frac{N}{m}}$, for the returned solutions $\hat{\boldsymbol{w}}$ and $\hat{\boldsymbol{\lambda}}$ it holds that
\begin{equation*}
\begin{aligned}
   &\max_{\boldsymbol{\lambda}\in \Lambda}\mathbb{E}[F(\hat{\boldsymbol{w}},\boldsymbol{\lambda} )] -\min_{\bm{w}\in\mathcal{W}} \mathbb{E}[F(\boldsymbol{w} ,\hat{\boldsymbol{\lambda}} )]\nonumber \\ & \quad \leq O\Big{(}\frac{D_{\mathcal{W}}^2+G_{w}^2}{\sqrt{T}} +\frac{D_{\Lambda}^2}{T^{3/8}}  +\frac{G_{\lambda}^2}{m^{1/2}T^{3/8}} +\frac{\sigma_{\lambda}^2}{m^{3/2}T^{3/8}}\nonumber \\ & \quad \quad + \frac{\sigma_w^2+\Gamma}{m\sqrt{T} }\Big{)}.
\end{aligned}
\end{equation*}
\end{theorem} 
\begin{proof}
The proof is provided in Appendix \ref{proof: theorem1}.
\end{proof}
The primal-dual duality gap found in Theorem \ref{theorem1} is composed of three main terms. That is, the first term $\frac{D_{\mathcal{W}}^2+G_{w}^2}{\sqrt{T}}$ arises from the smoothness of local functions and the boundedness of the gradients. The decay is at a rate $\mathcal{O}(1/T^{1/2})$, which matches the original agnostic federated learning using the SGD steps. The second collection of terms $\frac{D_{\Lambda}^2}{T^{3/8}}, \frac{G_{\lambda}^2}{m^{1/2}T^{3/8}},$ and $\frac{\sigma_{\lambda}^2}{m^{3/2}T^{3/8}}$ comes from the periodic update of the dual variable $\boldsymbol{\lambda}$ at every $\tau =  \frac{T^{1/4}}{\sqrt{m}}$ local step. These terms contribute to a decay $\mathcal{O}(1/T^{3/8})$, which reflects the penalty incurred from infrequent updates of the dual variable, and the communication savings from the decrease in communication rounds to $\mathcal{O}(T^{3/4})$. Finally, the term $\frac{\sigma_w^2+\Gamma}{m\sqrt{T} }$ captures the effects of heterogeneity in terms of the variance of stochastic gradients and the measure of dissimilarity between the gradients of different clients. The term is divided by the $m$ participating clients and decays at $\mathcal{O}(1/T^{1/2})$. By choosing 
\begin{align}
    \tau =  \frac{T^{1/4}}{\sqrt{m}},\,\,\, \eta = \frac{1}{4L \sqrt{T}},\,\,\, \gamma = \frac{1}{T^{5/8}},\,\,\, \mu = 2L \sqrt{\frac{N}{m}},\nonumber
\end{align}
the sources of error are balanced such that the rate is dominated by $\mathcal{O}(1/T^{3/8})$.
\section{EXPERIMENTS}\label{experiments}
\subsection{EXPERIMENTAL SETUP}
We consider a system model with $N = 30$ clients and a server. In each communication round, $m = 20$ clients are sampled to participate in the learning. Each client processes $\tau = 10$ local iterations before communicating with the server. We report the averaged results of the experiments from $10$ Monte Carlo runs.
\subsubsection{BASELINES} Our proposed algorithm \textit{DRDM} is compared to the following baselines.
\begin{itemize}
    \item \textit{SCAFFOLD:} Stochastic controlled averaging
for FL. \textit{SCAFFOLD} aims to correct the client drift by correcting the local update using an estimate of the client drift. \cite{scaffold}
    \item \textit{SCAFF-PD:} Accelerated primal-dual FL algorithm with bias-corrected local steps. \textit{SCAFF-PD} leverages the accelerated primal-dual algorithm and corrects the bias caused by local steps using \textit{SCAFFOLD}-like control variates \cite{scaffpd}.
    \item \textit{FedAvg:} Federated averaging allows participating clients to perform local stochastic gradient descent (SGD) update steps on their local data before
communicating with the server for model aggregation \cite{fl}.
    \item \textit{DRFA:} The distributionally robust federated averaging algorithm solves the problem in (\ref{dro}) by allowing participating clients to perform multiple local SGD steps at the clients’ side, while the dual variable $\lambda$ is only updated periodically at every communication round \cite{drfa}.
\end{itemize}
\subsubsection{DATASETS} We used three datasets in our experiments, namely MNIST \cite{mnist}, Fashion-MNIST \cite{fashionmnist}, and Kuzushiji-MNIST \cite{kuzushijimnist}. The MNIST dataset consists of $28 \times 28$ grayscale handwritten digits ranging from 0 to 9 images, with 60K samples for training and 10K for testing. Fashion-MNIST has a similar format (28x28 grayscale, 10 classes) but contains images of Zalando fashion articles. Kuzushiji-MNIST serves as a drop-in replacement of MNIST and Fashion-MNIST datasets, where the authors chose one character to represent each of the 10 rows of the Hiragana writing system, using Kuzushiji, which is a Japanese cursive writing style.\\
\subsubsection{MODELS} We aim to solve a classification task by using different model architectures for different datasets. This allows us to have more generalized results when comparing our algorithm with the different baselines. With the MNIST dataset, we apply a linear model with input size $28 \times 28$ (flattened input data size) and output size of $10$ (number of classes). For the Fashion-MNIST dataset, we fine-tune the last layer of the pre-trained ResNet18 model \cite{resnet18}. By keeping the model backbone fixed, we preserve its low-level feature detection learned from large-scale datasets, leading to faster and more stable training and preventing overfitting when dealing with relatively small datasets such as Fashion-MNIST. The trained layer has an input size of $512$ and an output size of $10$. Finally, for the Kuzushiji-MNIST dataset, we employ a convolutional neural network (CNN) architecture consisting of a sequence of convolutional and pooling layers. The network begins with a convolutional layer that utilizes $3 \times 3$ kernels and $16$ output channels, followed by a max-pooling operation. This is succeeded by a second convolutional layer with $3 \times 3$ kernels and $32$ output channels, again followed by max-pooling. The resulting feature maps are then flattened and passed through a fully connected (FC) layer with an input dimension of $32 \times 7 \times 7$ and an output dimension of $500$, followed by a rectified linear unit (ReLU) activation. The final FC layer outputs a $10$-dimensional vector, corresponding to the number of target classes. All experiments are carried out using a batch size of $32$.\\
\subsubsection{DATA PARTITIONING} We generate non-IID data by partitioning the total dataset among clients using stratified sampling guided by Zipf and Dirichlet distributions. The data heterogeneity is studied from two aspects: i) heterogeneity of local dataset sizes, i.e., clients hold local datasets with varying numbers of samples, and ii) heterogeneity of class availability, i.e., the local datasets of clients contain different numbers of samples per class. Heterogeneity in terms of local dataset size is achieved by partitioning the samples among the clients using the Zipf distribution, where each client $i$ receives a local dataset of size $d_i = d / i^{\sigma} \sum_{j \in \mathcal{N}} j^{-\sigma}$ samples \cite{zipf_dist_paper}. The parameter $\sigma$ controls the number of samples that each client gets. While $\sigma = 0$ results in clients having approximately equal dataset sizes, increasing $\sigma$ leads to greater heterogeneity in dataset sizes, concentrating more samples on fewer clients. Dataset class availability is controlled by the Dirichlet distribution, with a concentration parameter $ \alpha \in (0, \infty]$ that controls the per-class sample distribution among clients \cite{dirichlet_dist_paper}. Setting $\alpha = 0$ means that each client has samples from a single class, increasing $\alpha$ allows clients to have samples from more classes, and as $\alpha \to \infty$ , local datasets contain samples that are uniformly distributed from all classes.  

\begin{figure*}[t]
\centering
\includegraphics[width=1\textwidth,  height=6cm, trim={0cm 0cm 0cm 0cm}
]{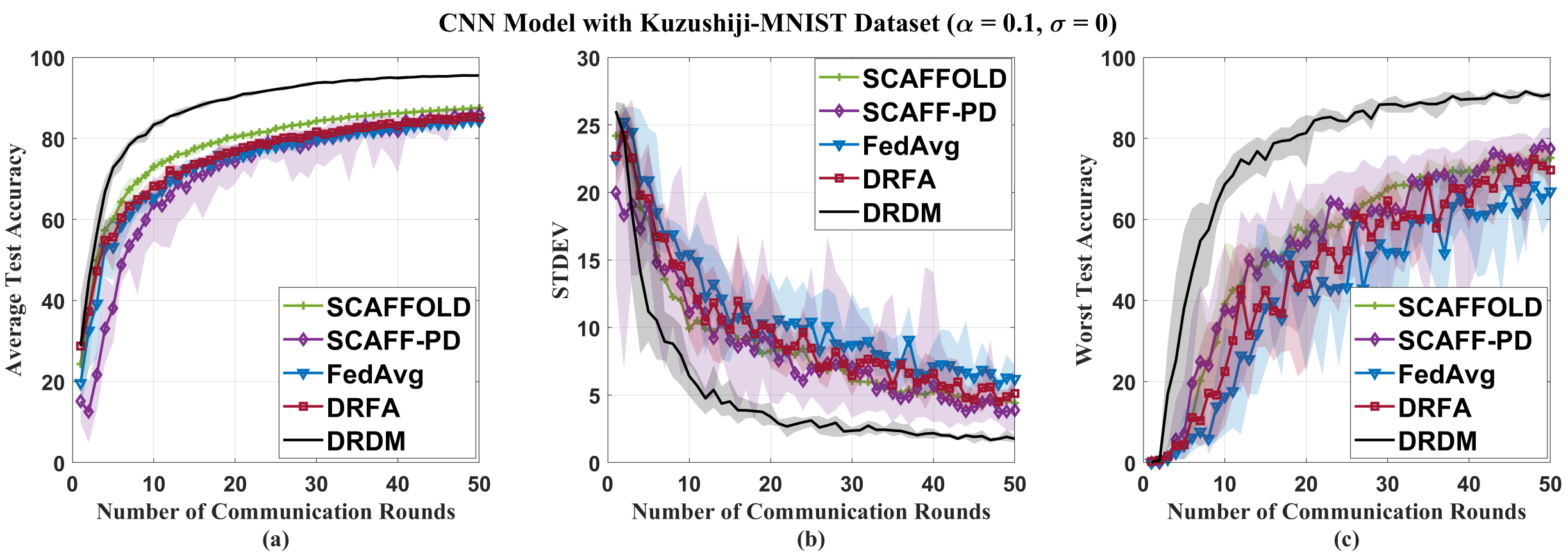}	
\caption{Results of \textit{DRDM} compared to the other baselines using CNN model with Kuzushiji-MNIST dataset with non-IID setting ($\alpha = 0.1$ and $\sigma = 0$). (a) Average test accuracy, (b) standard deviation values, and (c) worst-case test accuracy experienced by the different algorithms.}
\label{fig:cnn_01_0}
\end{figure*} 
\begin{figure*}[t]
\centering
\includegraphics[width=1\textwidth,  height=6cm, trim={0cm 0cm 0cm 0cm}]{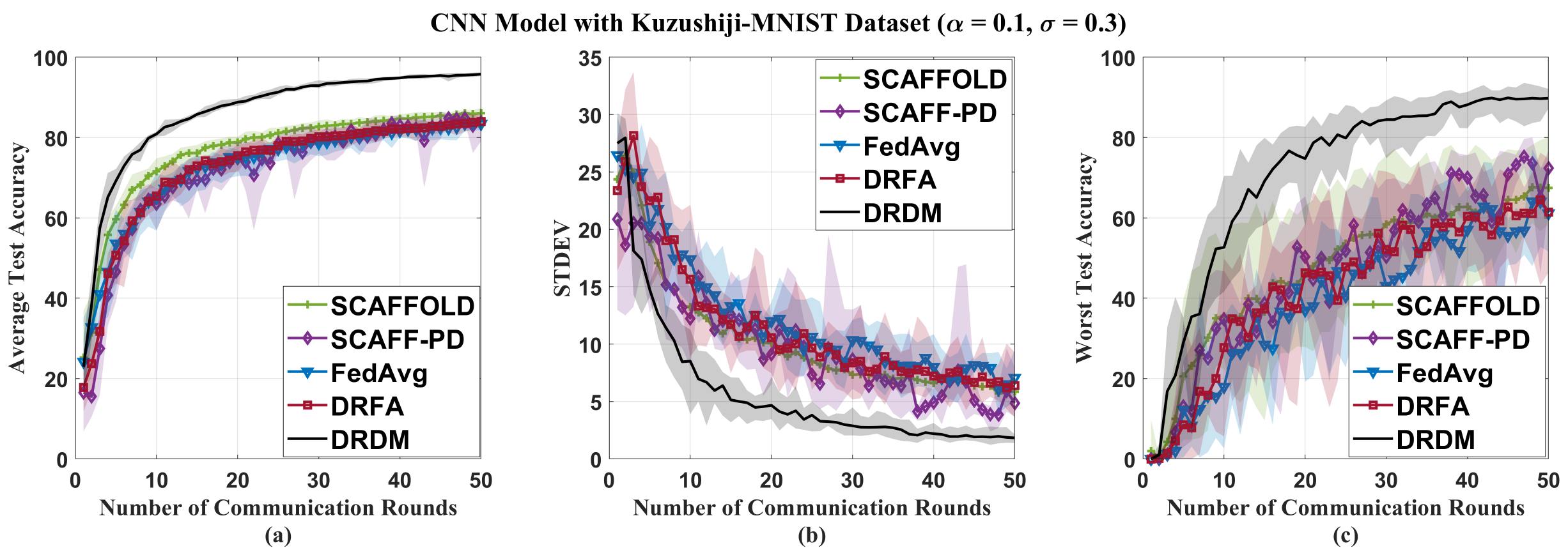}	
\caption{Results of \textit{DRDM} compared to the other baselines using CNN model with Kuzushiji-MNIST dataset with non-IID setting ($\alpha = 0.1$ and $\sigma = 0.3$). (a) Average test accuracy, (b) standard deviation values, and (c) worst-case test accuracy experienced by the different algorithms.}
\label{fig:cnn_01_03}
\end{figure*}

\begin{table*}[t]
\centering
\caption{Results of the different algorithms using CNN model and Kuzushiji-MNIST dataset, with $\sigma = 0$ and $\alpha \in \{0.1, 0.3, 0.5, 0.7\}$.}
\begin{tabular}{|
>{}c l|
>{}c 
>{}c 
>{}c 
>{}c |}
\hline
\multicolumn{2}{|c|}{Dataset} &
  \multicolumn{4}{c|}{Kuzushiji-MNIST} \\ \hline
\multicolumn{2}{|c|}{{($\alpha | \sigma = 0$)}} &
  \multicolumn{1}{P{2cm}|}{0.1} &
  \multicolumn{1}{P{2cm}|}{0.3} &
  \multicolumn{1}{P{2cm}|}{0.5} &
  \multicolumn{1}{P{2cm}|}{0.7} \\ \hline
\multicolumn{2}{|c|}{Model type} &
  \multicolumn{4}{c|}{CNN} \\ \hline
\multicolumn{1}{|c|}{} &
   \multicolumn{1}{l|}{{Average test accuracy}} &
  \multicolumn{1}{c|}{\textbf{95.44}} &
  \multicolumn{1}{c|}{\textbf{95.10}} &
  \multicolumn{1}{c|}{\textbf{95.27}} &
  \multicolumn{1}{c|}{\textbf{95.68}} \\ \cline{2-6} 
\multicolumn{1}{|c|}{} &
  \multicolumn{1}{l|}{{Worst-case test accuracy}} &
  \multicolumn{1}{c|}{\textbf{91.58}} &
  \multicolumn{1}{c|}{\textbf{92.73}} &
  \multicolumn{1}{c|}{\textbf{92.75}} &
  \multicolumn{1}{c|}{\textbf{92.84}} \\ \cline{2-6} 
\multicolumn{1}{|c|}{\multirow{-3}{*}{DRDM}} &
   \multicolumn{1}{l|}{{Standard deviation}} &
  \multicolumn{1}{c|}{\textbf{1.68}} &
 \multicolumn{1}{c|}{\textbf{1.26}} &
 \multicolumn{1}{c|}{\textbf{1.34}} &
  \multicolumn{1}{c|}{\textbf{1.37}} \\ \hline
\multicolumn{1}{|c|}{} &
   \multicolumn{1}{l|}{{Average test accuracy}} &
  \multicolumn{1}{c|}{85.27} &
  \multicolumn{1}{c|}{85.16} &
  \multicolumn{1}{c|}{85.78} &
  86.32 \\ \cline{2-6} 
\multicolumn{1}{|c|}{} &
  \multicolumn{1}{l|}{{Worst-case test accuracy}} &
  \multicolumn{1}{c|}{74.79} &
  \multicolumn{1}{c|}{78.80} &
  \multicolumn{1}{c|}{79.96} &
  81.76 \\ \cline{2-6} 
\multicolumn{1}{|c|}{\multirow{-3}{*}{DRFA}} &
   \multicolumn{1}{l|}{{Standard deviation}} &
  \multicolumn{1}{c|}{4.51} &
  \multicolumn{1}{c|}{2.79} &
  \multicolumn{1}{c|}{2.85} &
  2.34 \\ \hline
\multicolumn{1}{|c|}{} &
   \multicolumn{1}{l|}{{Average test accuracy}} &
  \multicolumn{1}{c|}{87.54} &
  \multicolumn{1}{c|}{86.46} &
  \multicolumn{1}{c|}{86.95} &
  87.08 \\ \cline{2-6} 
\multicolumn{1}{|c|}{} &
  \multicolumn{1}{l|}{{Worst-case test accuracy}} &
  \multicolumn{1}{c|}{75.21} &
  \multicolumn{1}{c|}{80.36} &
  \multicolumn{1}{c|}{80.47} &
  81.96 \\ \cline{2-6} 
\multicolumn{1}{|c|}{\multirow{-3}{*}{SCAFFOLD}} &
   \multicolumn{1}{l|}{{Standard deviation}} &
  \multicolumn{1}{c|}{4.41} &
  \multicolumn{1}{c|}{2.77} &
  \multicolumn{1}{c|}{3.00} &
  2.43 \\ \hline
\multicolumn{1}{|c|}{} &
  \multicolumn{1}{l|}{{Average test accuracy}} &
  \multicolumn{1}{c|}{85.62} &
  \multicolumn{1}{c|}{87.15} &
  \multicolumn{1}{c|}{87.59} &
  87.62 \\ \cline{2-6} 
\multicolumn{1}{|c|}{} &
  \multicolumn{1}{l|}{{Worst-case test accuracy}} &
  \multicolumn{1}{c|}{78.18} &
  \multicolumn{1}{c|}{81.76} &
  \multicolumn{1}{c|}{82.19} &
  83.12 \\ \cline{2-6} 
\multicolumn{1}{|c|}{\multirow{-3}{*}{SCAFF-PD}} &
  \multicolumn{1}{l|}{{Standard deviation}} &
  \multicolumn{1}{c|}{3.79} &
  \multicolumn{1}{c|}{2.47} &
  \multicolumn{1}{c|}{2.42} &
  2.19 \\ \hline
\multicolumn{1}{|c|}{} &
  \multicolumn{1}{l|}{{Average test accuracy}} &
  \multicolumn{1}{c|}{84.12} &
  \multicolumn{1}{c|}{85.30} &
  \multicolumn{1}{c|}{85.86} &
  86.16 \\ \cline{2-6} 
\multicolumn{1}{|c|}{} &
  \multicolumn{1}{l|}{{Worst-case test accuracy}} &
  \multicolumn{1}{c|}{68.42} &
  \multicolumn{1}{c|}{78.63} &
  \multicolumn{1}{c|}{79.83} &
  80.67 \\ \cline{2-6} 
\multicolumn{1}{|c|}{\multirow{-3}{*}{Fed-Avg}} &
   \multicolumn{1}{l|}{{Standard deviation}} &
  \multicolumn{1}{c|}{5.83} &
  \multicolumn{1}{c|}{3.11} &
  \multicolumn{1}{c|}{2.85} &
  2.48 \\ \hline
\end{tabular}
\label{table:cnn_01_0}
\end{table*}

\begin{table*}[t]
\caption{Results of the different algorithms using CNN model and Kuzushiji-MNIST dataset, with $\alpha = 0.1$ and $\sigma \in \{0.3, 0.5, 0.7\}$.}
\centering
\begin{tabular}{|
>{}c l|
>{}c 
>{}c 
>{}c |}
\hline
\multicolumn{2}{|c|}{Dataset} &
  \multicolumn{3}{c|}{Kuzushiji-MNIST} \\ \hline
\multicolumn{2}{|c|}{{($\sigma | \alpha = 0.1$)}} &
  \multicolumn{1}{P{2cm}|}{0.3} &
  \multicolumn{1}{P{2cm}|}{0.5} &
  \multicolumn{1}{P{2cm}|}{0.7} \\ \hline
\multicolumn{2}{|c|}{Model type} &
  \multicolumn{3}{c|}{CNN} \\ \hline
\multicolumn{1}{|c|}{} &
  \multicolumn{1}{l|}{{Average test accuracy}} &
  \multicolumn{1}{c|}{\textbf{95.42}} &
  \multicolumn{1}{c|}{\textbf{94.26}} &
   \multicolumn{1}{c|}{\textbf{93.32}} \\ \cline{2-5} 
\multicolumn{1}{|c|}{} &
\multicolumn{1}{l|}{{Worst-case test accuracy}} &
  \multicolumn{1}{c|}{\textbf{89.87}} &
  \multicolumn{1}{c|}{\textbf{88.44}} &
  \multicolumn{1}{c|}{\textbf{86.83}} \\ \cline{2-5} 
\multicolumn{1}{|c|}{\multirow{-3}{*}{DRDM}} &
   \multicolumn{1}{l|}{{Standard deviation}} &
  \multicolumn{1}{c|}{\textbf{1.92}} &
  \multicolumn{1}{c|}{\textbf{2.39}} &
  \multicolumn{1}{c|}{\textbf{2.64}} \\ \hline
\multicolumn{1}{|c|}{} &
   
 \multicolumn{1}{l|}{{Average test accuracy}} &
  \multicolumn{1}{c|}{83.71} &
  \multicolumn{1}{c|}{80.98} &
  78.86 \\ \cline{2-5} 
\multicolumn{1}{|c|}{} &
  \multicolumn{1}{l|}{{Worst-case test accuracy}} &
  \multicolumn{1}{c|}{64.46} &
  \multicolumn{1}{c|}{63.63} &
  61.42 \\ \cline{2-5} 
\multicolumn{1}{|c|}{\multirow{-3}{*}{DRFA}} &
   \multicolumn{1}{l|}{{Standard deviation}} &
  \multicolumn{1}{c|}{6.23} &
  \multicolumn{1}{c|}{6.95} &
  6.94 \\ \hline
\multicolumn{1}{|c|}{} &
    \multicolumn{1}{l|}{{Average test accuracy}} &
  \multicolumn{1}{c|}{86.03} &
  \multicolumn{1}{c|}{83.29} &
  81.04 \\ \cline{2-5} 
\multicolumn{1}{|c|}{} &
  \multicolumn{1}{l|}{{Worst-case test accuracy}} &
  \multicolumn{1}{c|}{67.62} &
  \multicolumn{1}{c|}{65.34} &
  62.62 \\ \cline{2-5} 
\multicolumn{1}{|c|}{\multirow{-3}{*}{SCAFFOLD}} &
   \multicolumn{1}{l|}{{Standard deviation}} &
  \multicolumn{1}{c|}{5.71} &
  \multicolumn{1}{c|}{7.35} &
  6.85 \\ \hline
\multicolumn{1}{|c|}{} &
    \multicolumn{1}{l|}{{Average test accuracy}} &
  \multicolumn{1}{c|}{84.63} &
  \multicolumn{1}{c|}{84.56} &
  83.71 \\ \cline{2-5} 
\multicolumn{1}{|c|}{} &
 \multicolumn{1}{l|}{{Worst-case test accuracy}} &
  \multicolumn{1}{c|}{75.23} &
  \multicolumn{1}{c|}{74.15} &
  72.66 \\ \cline{2-5} 
\multicolumn{1}{|c|}{\multirow{-3}{*}{SCAFF-PD}} &
 \multicolumn{1}{l|}{{Standard deviation}} &
  \multicolumn{1}{c|}{3.88} &
  \multicolumn{1}{c|}{4.25} &
  5.18 \\ \hline
\multicolumn{1}{|c|}{} &
    \multicolumn{1}{l|}{{Average test accuracy}} &
  \multicolumn{1}{c|}{83.44} &
  \multicolumn{1}{c|}{80.65} &
  78.92 \\ \cline{2-5} 
\multicolumn{1}{|c|}{} &
  \multicolumn{1}{l|}{{Worst-case test accuracy}} &
  \multicolumn{1}{c|}{64.07} &
  \multicolumn{1}{c|}{61.25} &
  57.99 \\ \cline{2-5} 
\multicolumn{1}{|c|}{\multirow{-3}{*}{Fed-Avg}} &
    \multicolumn{1}{l|}{{Standard deviation}} &
  \multicolumn{1}{c|}{6.07} &
  \multicolumn{1}{c|}{7.85} &
  7.46 \\ \hline
\end{tabular}
\label{table:cnn_01_03}
\end{table*}

\subsection{RESULTS AND DISCUSSION}
\subsubsection{IMPACT OF DATA HETEROGENEITY}Figures~\ref{fig:cnn_01_0} and~\ref{fig:cnn_01_03} present the performance of \textit{DRDM} in comparison to the baseline methods on the Kuzushiji-MNIST dataset using the CNN model architecture. Both experiments are carried out with $\alpha = 0.1$, while $\sigma$ is set to zero in Figure~\ref{fig:cnn_01_0} and to $0.3$ in Figure~\ref{fig:cnn_01_03}. 
Figure~\ref{fig:cnn_01_0}a presents the average test accuracy over communication rounds, where our proposed algorithm, \textit{DRDM}, consistently outperforms all baseline methods. Figure~\ref{fig:cnn_01_0}b shows the standard deviation across clients, which reflects the fairness of the model performance. In particular, \textit{DRDM} achieves the lowest standard deviation, indicating improved robustness and fairer performance among clients. In Figure~\ref{fig:cnn_01_0}c, a similar trend is observed, with \textit{DRDM} attaining the highest worst-case test accuracy, further supporting its ability to provide robust performance under non-IID data distributions. Additionally, \textit{FedAvg} consistently achieves the lowest performance across all metrics, highlighting the limitations of simple averaging in heterogeneous data settings and the need for more sophisticated aggregation strategies. 
In a more heterogeneous setting shown in Figure~\ref{fig:cnn_01_03} with $\sigma = 0.3$, \textit{DRDM} continues to outperform all baselines across all performance indicators. Table~\ref{table:cnn_01_0} explores the heterogeneity of the class availability ($\alpha \in \{0.1, 0.3, 0.5, 0.7\}$) under uniform client data sizes ($\sigma = 0$), while Table~\ref{table:cnn_01_03} investigates the performance with class availability level ($\alpha = 0.1$) and increasing dataset size heterogeneity ($\sigma \in \{0.3, 0.5, 0.7\}$). Across all configurations, \textit{DRDM} consistently achieves the best performance in terms of average accuracy, worst-case test accuracy, and fairness, with the highest performance results highlighted in bold.
\subsubsection{IMPACT OF THE NUMBER OF LOCAL STEPS} Figure \ref{cnn_local_steps} shows the worst-case test accuracy obtained by \textit{DRDM} for various values of local steps $\tau$, with respect to the number of communication rounds $s$. We notice that a higher value of local steps results in better performance with fewer communication rounds. In particular, for $\tau = 30$, \textit{DRDM} needs $s = 25$ communication rounds to reach the $85\%$ worst-case test accuracy, while for $\tau = 5$, $10$ and $20$, it requires $s = 48$, $41$ and $32$, respectively.
\subsubsection{IMPACT OF CHANNEL CAPACITY}Although \textit{DRDM} achieves higher performance when increasing the number of local steps, as illustrated in Figure \ref{cnn_local_steps}, it can come at the cost of more data processing at the client side. Consequently, the processing energy can become a bottleneck for energy-constrained devices. The total energy cost at each client comprises both processing energy $E^p$ and transmission energy $E^t$, and is defined as follows
\begin{align}  \label{equ:energy}
    E = T \times \tau \sum_{i=1}^m E_i^p + T \times \sum_{i=1}^m E_i^t,
\end{align}
where $E_i^p$ and $E_i^t$ denote the $i^{th}$ client's processing energy per update step and transmission energy, respectively. As seen in \eqref{equ:energy}, the processing energy depends both on the number of local steps and the number of communication rounds, while the transmission energy is influenced by the number of communication rounds, the experienced signal-to-noise ratio (SNR), and the available bandwidth between the clients and the server. Figure \ref{cnn_energy} illustrates how the SNR and the bandwidth affect the optimal number of local steps required to achieve a target worst-case test accuracy of 80\% with minimal total energy cost. In low-SNR and bandwidth-constrained conditions, the optimal number of local steps increases (e.g., $\tau = \{25, 30\}$), since more local computation reduces the number of required communication rounds, which is beneficial when transmission energy dominates the total energy cost. As the SNR increases, transmission becomes more efficient and the optimal number of local steps decreases (e.g., $\tau = 15$ or lower for all bandwidths at $SNR = 20\,\mathrm{dB}$). In this high-SNR regime, transmission energy becomes less significant, allowing clients to reduce local computation in favor of lower processing energy, which now constitutes the dominant part of the total energy cost.
\section{CONCLUSION}\label{conclusion}
In this work, we introduced \textit{DRDM}, an algorithm that enhances the robustness and fairness of federated learning models in heterogeneous settings by combining a DRO formulation with client drift mitigation via dynamic regularization.
We provided a rigorous convergence analysis for convex objectives, accounting for multiple local SGD steps, dynamic regularization, partial client participation, and periodic dual updates via randomized snapshotting. We bounded the duality gap and derived a convergence rate of $\mathcal{O}(1/T^{3/8})$, where $T$ represents the total number of iterations. Empirically, across three benchmark datasets with different ML model architectures and under varying non-IID scenarios, \textit{DRDM} achieves convergence with a reduced number of communication rounds compared to baselines. Moreover, it achieves a substantial improvement in the worst-case test accuracy compared to baselines across different heterogeneity settings. Finally, by characterizing the interplay between the SNR, bandwidth, and energy costs, we demonstrated how \textit{DRDM} can be adaptively tuned, through the number of local update steps, to meet a target worst-case test accuracy with minimal total energy costs. 

\begin{figure}[t]
\centering
\includegraphics[width=0.45\textwidth,  height=7cm]{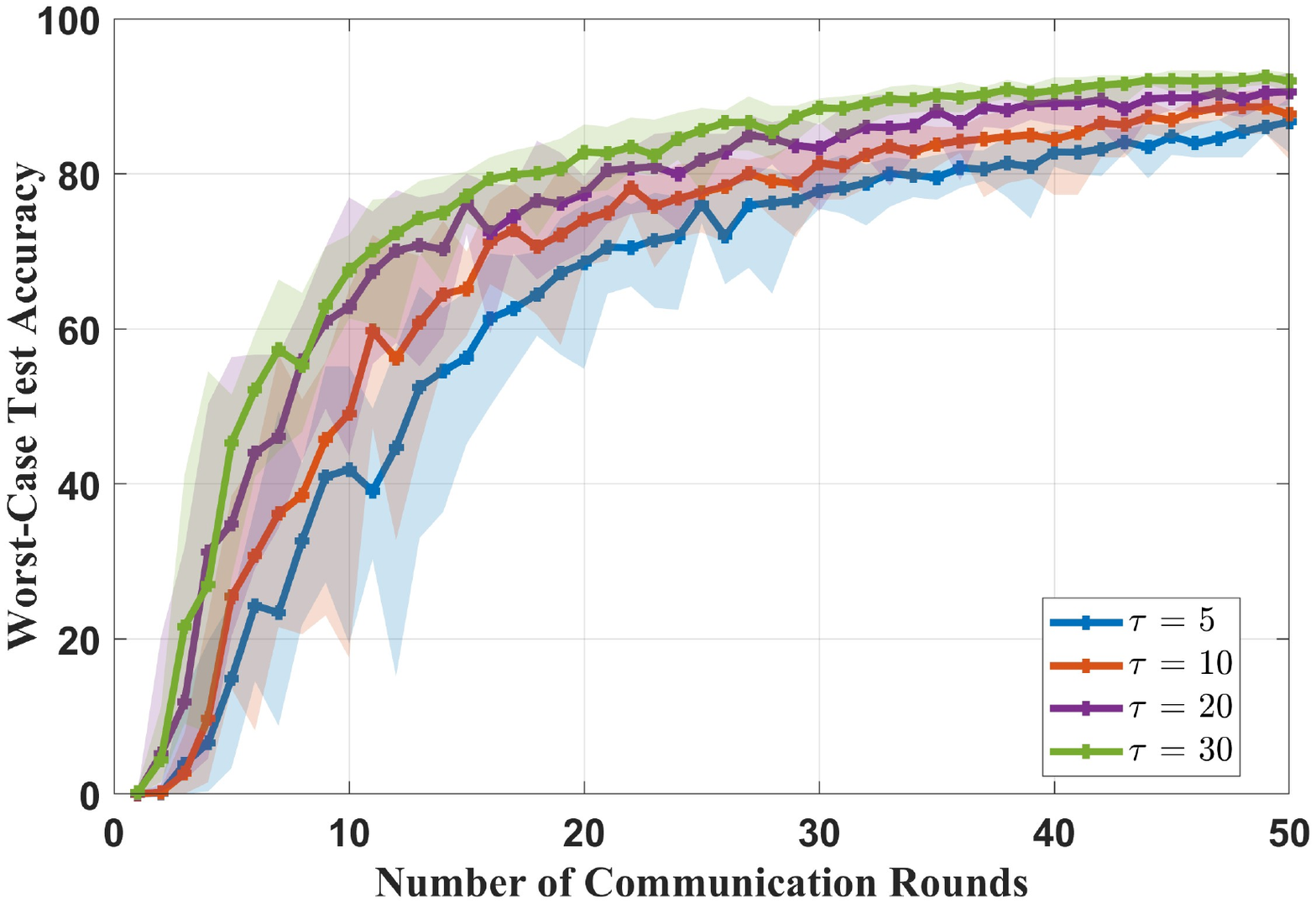}	
\caption{Worst-case test accuracy versus the number of communication rounds, with respect to different values of local steps $\tau$.}
\label{cnn_local_steps}
\end{figure} 
\begin{figure}[t]
\centering
\includegraphics[width=0.45\textwidth,  height=7cm]{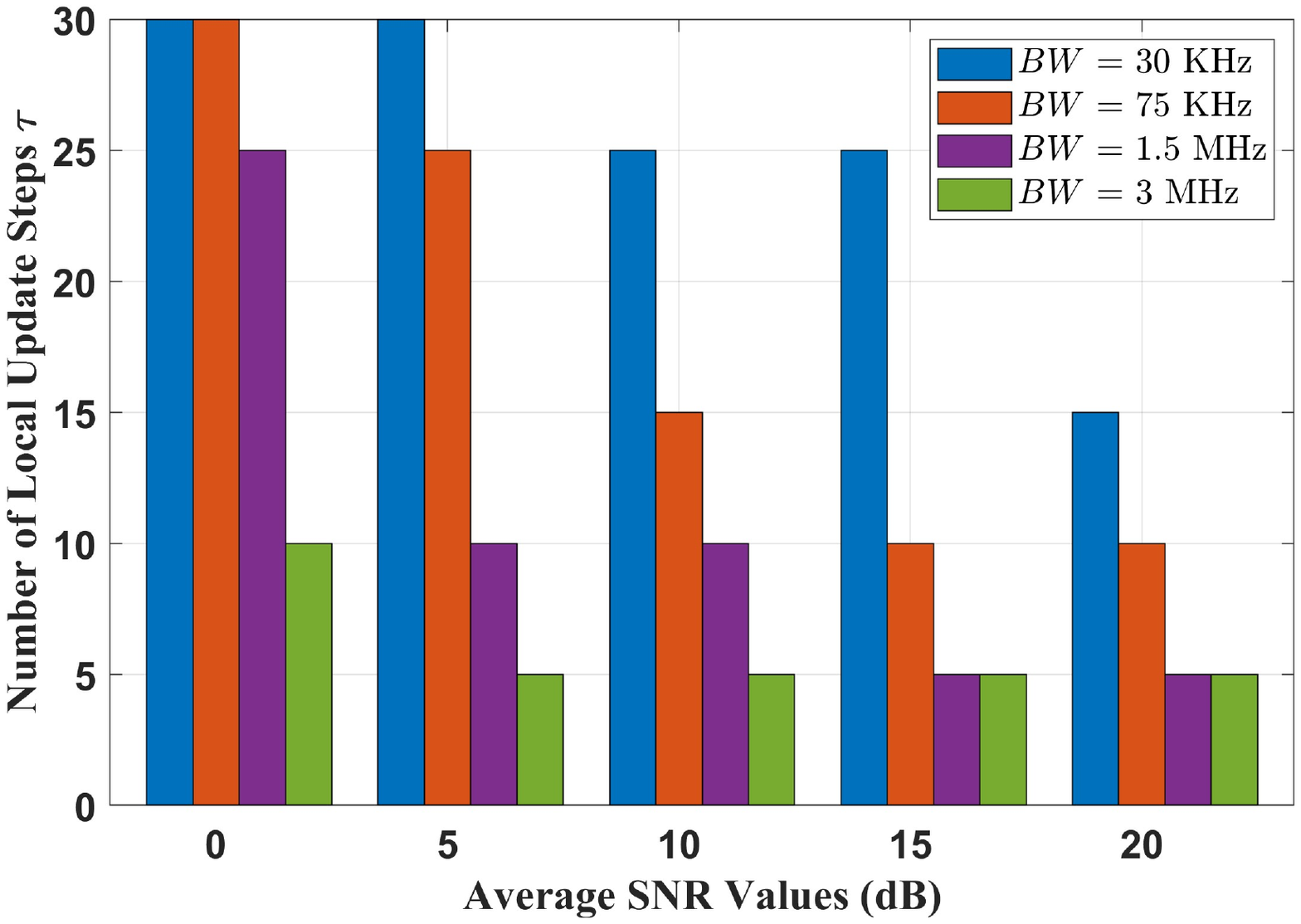}	
\caption{Number of local steps to achieve 80$\%$ worst-case test accuracy with minimum total energy cost, for different SNR values and bandwidth choices.}
\label{cnn_energy}
\end{figure} 
\bibliographystyle{IEEEtran}  
\bibliography{IEEEabrv,refs}
\section*{APPENDIX} \label{sec: proof DRDM convex}
\setcounter{lemma}{0}
\setcounter{theorem}{0}
In this section, we present the proof of Theorem~\ref{theorem1}, which establishes the convergence of the \textit{DRDM} algorithm in the convex-linear setting. In addition, we provide supplementary experimental results to further support our findings.
\subsection{OVERVIEW OF THE PROOF}\label{overview_proof}
We analyze the one-step progression of the virtual iterates, $\bm{w}^{(t+1)}$ and $\boldsymbol{\lambda}^{(s+1)}$, as part of our proof approach. Compared to the fully synchronous primal-dual methods for min-max optimization, our analysis is more involved, considering the periodic decoupled updates, clients sampling, and drift terms introduced to the local updates of the primal model.  

We begin by examining a single iteration of $\bm{w}$, which satisfies the following update rule,
\begin{align}
&\mathbb{E}\|\bm{w}^{(t+1)} - \bm{w} \|^2 \nonumber \\ & \,\,\, \leq \mathbb{E}\|\bm{w}^{(t)} - \bm{w} \|^2 + 2\eta^2 G_w^2 + 4 \eta^2 \mathbb{E}\left[\|\bm{u}^{(t)} - \Bar{\bm{u}}^{(t)}\|^2\right] \nonumber \\ & \,\,\,\quad + L\eta \mathbb{E}\Big[ \delta^{(t)} \Big] -2\eta\mathbb{E}\left[F(\bm{w} ^{(t)},\boldsymbol{\lambda}^{(\lfloor{\frac{t}{\tau}}\rfloor)})  - F(\bm{w} ,\boldsymbol{\lambda}^{(\lfloor{\frac{t}{\tau}}\rfloor)})\right]  \nonumber \\ & \,\,\, \quad+ 4\eta^2 \mu^2 \mathbb{E} \Big[\beta^{(t)} \Big] + 4\mathbb{E} \Big[\varphi^{(t)} \Big].
\end{align}
Next, we focus on bounding the terms $\delta^{(t)}$, $\beta^{(t)}$, and $\varphi^{(t)}$, which represent the different deviations between the local and virtual global models. First, we bound $\delta^{(t)}$ using Lemma~\ref{lemma: deviation1}, obtaining
\begin{align}
    &\frac{1}{T}\sum_{t=0}^{T}\mathbb{E}\left[\delta^{(t)}\right]  
    \nonumber \\ & \quad \leq 18\eta^2\tau^2 \Big( \sigma_w^2 + \frac{\sigma_w^2}{m} + \Gamma        + 2G_w^2 + 2\mu^2 D_{\mathcal{W}}^2 \Big),
\end{align}
where the bound is characterized by the variance of the stochastic gradient, gradient dissimilarity, gradient boundedness, and domain boundedness. Next, we bound $\beta^{(t)}$ using Lemma~\ref{lemma: deviation2}    
\begin{align}
     \frac{1}{T}\sum_{t=0}^{T}\mathbb{E}\Big[ \beta^{(t)} \Big] \leq  6\eta^2\tau^2  \Big(\sigma_w^2 + 4 G_w^2 \Big).
\end{align}
Finally, we bound $\varphi^{(t)}$, which is a combination of $\delta^{(t)}$ and $\beta^{(t)}$ using Lemma~\ref{lemma: deviation3} as follows
\begin{align}
     &\frac{1}{T}\sum_{t=0}^{T}\mathbb{E}\Big[ \varphi^{(t)} \Big] \nonumber \\  &\quad  \leq   72\eta^2\tau^2  \left(\sigma_w^2+\frac{\sigma_w^2}{m} + \Gamma + 2G_w^2 + 2\mu^2 D_{\mathcal{W}}^2 \right) \nonumber \\& \quad \quad +  4\eta^2 \frac{\sigma_w^2}{m} + 16 \eta^2G_w^2+ 6\eta^4\tau^2  \mu^2\Big(\sigma_w^2 + 4 G_w^2 \Big).
\end{align}
Now we switch to the one iteration analysis on $\boldsymbol{\lambda}$
\begin{align}
 &\mathbb{E} \|\boldsymbol{\lambda}^{(s+1)} - \boldsymbol{\lambda} \|^2 \nonumber\\ &\quad \leq \mathbb{E}\|\boldsymbol{\lambda}^{(s)}- \boldsymbol{\lambda} \|^2 + \mathbb{E}\|\bar{\Delta}_{t}\|^2 +  \mathbb{E}\|\Delta_{t} - \bar{\Delta}_{t}\|^2 \nonumber\\ &\quad \quad -2\gamma\sum_{t=s\tau+1}^{(s+1)\tau} \mathbb{E}[(F(\bm{w} ^{(t)},\boldsymbol{\lambda})-F(\bm{w} ^{(t)},\boldsymbol{\lambda}^{(s)}))]. 
\end{align}
The bound of the variance of $\Delta_s$ is obtained from Lemma 1 as follows
\begin{align}
         \mathbb{E}[\|\Delta_{s} - \bar{\Delta}_{s}\|^2] \leq \gamma^2\tau^2\frac{\sigma_{\lambda}^2}{m}. 
\end{align} 
Putting all pieces together and taking the telescoping sum will yield the result in Theorem~\ref{theorem1}. We next present the proofs of several technical lemmas that will be used in establishing Theorem~\ref{theorem1}.
\subsection{PROOF OF LEMMA \ref{lemma: deviation1}}\label{proof_deviation1}
Consider $s\tau \leq t\leq (s+1)\tau$. Recall that we only perform averaging based on a uniformly sampled subset of clients $\mathcal{D}^{(\lfloor{\frac{t}{\tau}}\rfloor)}$ chosen from $\mathcal{N}$. Following the updating rule, we have
{\begin{align}
     & \mathbb{E}[\delta^{(t)}] \nonumber\\
     &  \quad= \mathbb{E}\Big[\frac{1}{m}\sum_{i\in\mathcal{D}^{(\lfloor{\frac{t}{\tau}}\rfloor)}}\|\bm{w} ^{(t)}_i- \bm{w}^{(t)}\|^2\Big] \nonumber\\
     & \quad \leq  \mathbb{E}\Big[\frac{1}{m}\sum_{i\in\mathcal{D}^{(\lfloor{\frac{t}{\tau}}\rfloor)}}\mathbb{E}\Big\|\boldsymbol{{w}}^{({\lfloor{\frac{t}{\tau}}\rfloor}\tau)}-   \sum_{r=s\tau}^{t -1} \eta  \Big[  \nabla \ell_i(\bm{w} ^{(r)}_i;\xi_i^{(r)})\nonumber\\
     & \quad \quad-  {\nabla}f_i(\boldsymbol{w}_i^{({\lfloor{\frac{t}{\tau}}\rfloor}\tau)}) - \mu (\boldsymbol{{w}}^{({\lfloor{\frac{t}{\tau}}\rfloor}\tau)} - \boldsymbol{w}_{i}^{(r)}) \Big ]  \nonumber\\
    &\quad \quad - \Big(\boldsymbol{{w}}^{({\lfloor{\frac{t}{\tau}}\rfloor}\tau)}- \frac{1}{m}\sum_{j\in\mathcal{D}^{(\lfloor{\frac{t}{\tau}}\rfloor)}}\sum_{r=s\tau}^{t -1}\eta  \Big[\nabla \ell_{j}(\bm{w} ^{(r)}_{j};\xi_{j}^{(r)}) \nonumber\\
    &\quad \quad-  {\nabla}f_{j}(\boldsymbol{w}_{j}^{({\lfloor{\frac{t}{\tau}}\rfloor}\tau)}) - \mu (\boldsymbol{{w}}^{({\lfloor{\frac{t}{\tau}}\rfloor}\tau)} - \boldsymbol{w}_{j}^{(r)}) \Big]\Big) \Big\|^2\Big]\nonumber\\
    &\quad =  \mathbb{E}\Big[\frac{1}{m}\sum_{i\in\mathcal{D}^{(\lfloor{\frac{t}{\tau}}\rfloor)}} \Big\| \sum_{r=s\tau}^{t-1}\eta  \Big[ \nabla \ell_i(\bm{w} ^{(r)}_i;\xi_i^{(r)})\nonumber\\
    &\quad \quad-  {\nabla}f_{i}(\boldsymbol{w}_{i}^{({\lfloor{\frac{t}{\tau}}\rfloor}\tau)}) + \mu\boldsymbol{w}_{i}^{(r)} \Big] \nonumber\\
    &\quad \quad- \frac{1}{m}\sum_{j\in\mathcal{D}^{(\lfloor{\frac{t}{\tau}}\rfloor)}}\sum_{r=s\tau}^{t-1}\eta  \Big[ \nabla \ell_{j}(\bm{w} ^{(r)}_{j};\xi_{j}^{(r)}) -  {\nabla}f_{j}(\boldsymbol{w}_{j}^{({\lfloor{\frac{t}{\tau}}\rfloor}\tau)}) \nonumber\\
    &\quad \quad+ \mu\boldsymbol{w}_{j}^{(r)}  \Big]\Big\|^2 \Big]\nonumber\\
    & \quad\leq \mathbb{E}\Big[\frac{1}{m}\sum_{i\in\mathcal{D}^{(\lfloor{\frac{t}{\tau}}\rfloor)}}\eta^2\tau\sum_{r=s\tau}^{(s+1)\tau} \Big\|  \nabla \ell_i(\bm{w} ^{(r)}_i;\xi_i^{(r)}) \nonumber\\
    & \quad \quad-  {\nabla}f_{i}(\boldsymbol{w}_{i}^{({\lfloor{\frac{t}{\tau}}\rfloor}\tau)}) + \mu\boldsymbol{w}_{i}^{(r)}  \nonumber\\
    & \quad \quad-\frac{1}{m}\sum_{j\in\mathcal{D}^{(\lfloor{\frac{t}{\tau}}\rfloor)}} \Big[\nabla \ell_j(\bm{w} ^{(r)}_{j};\xi_{j}^{(r)})  \nonumber\\
    & \quad \quad-  {\nabla}f_{j}(\boldsymbol{w}_{j}^{({\lfloor{\frac{t}{\tau}}\rfloor}\tau)}) + \mu\boldsymbol{w}_{j}^{(r)}\Big]\Big\|^2\Big]\nonumber \\
     &\quad =  \eta^2\tau\mathbb{E}\Big[\frac{1}{m}\sum_{i\in\mathcal{D}^{(\lfloor{\frac{t}{\tau}}\rfloor)}}\sum_{r=s\tau}^{(s+1)\tau} \Big\|  \nabla \ell_i(\bm{w} ^{(r)}_i;\xi_i^{(r)})  \nonumber\\
    & \quad \quad-  {\nabla}f_{i}(\boldsymbol{w}_{i}^{({\lfloor{\frac{t}{\tau}}\rfloor}\tau)}) + \mu\boldsymbol{w}_{i}^{(r)} -\nabla f_i(\bm{w}^{(r)}_i ) \nonumber\\
    & \quad \quad + \nabla f_i(\bm{w}^{(r)}_i ) -\nabla f_i(\bm{w}^{(r)} ) \vphantom{-\frac{1}{m}\sum_{j\in\mathcal{D}^{(\lfloor{\frac{t}{\tau}}\rfloor)}} \nabla f_{j}(\bm{w}^{(r)}_{j})} +\nabla f_i(\bm{w}^{(r)}) \nonumber\\
    & \quad \quad-\frac{1}{m}\sum_{j\in\mathcal{D}^{(\lfloor{\frac{t}{\tau}}\rfloor)}} \nabla f_{j}(\bm{w}^{(r)}) 
      + \frac{1}{m}\sum_{j\in\mathcal{D}^{(\lfloor{\frac{t}{\tau}}\rfloor)}} \nabla f_{j}(\bm{w}^{(r)}) \nonumber\\
    &\quad \quad -\frac{1}{m}\sum_{j\in\mathcal{D}^{(\lfloor{\frac{t}{\tau}}\rfloor)}} \nabla f_{j}(\bm{w}^{(r)}_{j})+\frac{1}{m}\sum_{j\in\mathcal{D}^{(\lfloor{\frac{t}{\tau}}\rfloor)}} \nabla f_{j}(\bm{w}^{(r)}_{j})\nonumber\\
    &\quad \quad- \frac{1}{m}\sum_{j\in\mathcal{D}^{(\lfloor{\frac{t}{\tau}}\rfloor)}} \nabla \ell_{j}(\bm{w} ^{(r)}_{j};\xi_{j}^{(r)}) \nonumber\\
    &\quad \quad+ \frac{1}{m}\sum_{j\in\mathcal{D}^{(\lfloor{\frac{t}{\tau}}\rfloor)}}  {\nabla}f_{j}(\boldsymbol{w}_{j}^{({\lfloor{\frac{t}{\tau}}\rfloor}\tau)}) \nonumber\\
    &\quad \quad- \frac{1}{m}\sum_{j\in\mathcal{D}^{(\lfloor{\frac{t}{\tau}}\rfloor)}} \mu\boldsymbol{w}_{j}^{(r)}\Big\|^2\Big]. \label{eq: deviation 1}
      \end{align}}
Applying Jensen's inequality to split the norm yields
\begin{align}
      &\mathbb{E}[\delta^{(t)}]  \nonumber\\ 
      & \quad\leq  9\eta^2\tau\sum_{r=s\tau}^{(s+1)\tau}  \Big(\sigma_w^2 + L^2\mathbb{E}\Big[\frac{1}{m}\sum_{i\in\mathcal{D}^{(\lfloor{\frac{t}{\tau}}\rfloor)}}\Big\|\bm{w}^{(r)}_i - \bm{w}^{(r)}  \Big\|^2\Big]\nonumber\\ 
      & \quad \quad+ L^2\mathbb{E}\Big[\frac{1}{m}\sum_{j\in\mathcal{D}^{(\lfloor{\frac{t}{\tau}}\rfloor)}}\Big\|\bm{w}^{(r)}_{j} - \bm{w}^{(r)}  \Big\|^2\Big] \nonumber\\ 
      & \quad \quad +\mathbb{E}\Big[\frac{1}{m}\sum_{j\in\mathcal{D}^{(\lfloor{\frac{t}{\tau}}\rfloor)}}\left\|\nabla f_{i}(\bm{w}^{(r)}) - \nabla f_{j}(\bm{w} ^{(r)}) \right\|^2\Big]+\frac{\sigma_w^2}{m} \nonumber\\ 
      & \quad \quad +\mathbb{E}\Big[\frac{1}{m}\sum_{i\in\mathcal{D}^{(\lfloor{\frac{t}{\tau}}\rfloor)}}\Big\|{{\nabla}f_{i}(\boldsymbol{w}_{i}^{({\lfloor{\frac{t}{\tau}}\rfloor}\tau)}) } \Big\|^2\Big] \nonumber\\ 
      & \quad \quad+\mathbb{E}\Big[\frac{1}{m}\sum_{i\in\mathcal{D}^{(\lfloor{\frac{t}{\tau}}\rfloor)}}\Big\|\mu\boldsymbol{w}_{i}^{(r)} \Big\|^2\Big] \nonumber\\ 
      & \quad \quad +\mathbb{E}\Big[\frac{1}{m}\sum_{j\in\mathcal{D}^{(\lfloor{\frac{t}{\tau}}\rfloor)}}\Big\|{{\nabla}f_{j}(\boldsymbol{w}_{j}^{({\lfloor{\frac{t}{\tau}}\rfloor}\tau)}) } \Big\|^2\Big] \nonumber\\ 
      & \quad \quad+\mathbb{E}\Big[\frac{1}{m}\sum_{j\in\mathcal{D}^{(\lfloor{\frac{t}{\tau}}\rfloor)}}\Big\|\mu\boldsymbol{w}_{j}^{(r)} \Big\|^2\Big]  \Big)   \nonumber\\ 
      &\quad \leq  9\eta^2\tau  \sum_{r=s\tau}^{(s+1)\tau}\Big(\sigma^2_w + 2L^2\mathbb{E}[\delta^{(r)}]  \nonumber\\ 
      & \quad \quad + \Gamma + \frac{\sigma^2_w}{m} + 2G_w^2 + 2\mu^2 D_{\mathcal{W}}^2  \Big) \label{eq: deviation 3}.
\end{align}
Now, we sum \eqref{eq: deviation 3} over $t = s\tau$ to $(s+1)\tau$ to get
\begin{align}
    \sum_{t=s\tau}^{(s+1)\tau}\mathbb{E}[\delta^{(t)}] & \leq 9\eta^2\tau   \sum_{t=s\tau}^{(s+1)\tau}\sum_{r=s\tau}^{(s+1)\tau}\Big(\sigma^2_w + 2 L^2\mathbb{E}[\delta^{(r)}]   \nonumber\\ 
      & \quad + \Gamma +\frac{\sigma^2_w}{m} + 2 G_w^2 + 2\mu^2 D_{\mathcal{W}}^2   \Big) \nonumber\\
     & = 9\eta^2\tau^2 \sum_{r=s\tau}^{(s+1)\tau}\Big(\sigma_w^2 + 2L^2\mathbb{E}[\delta^{(r)}]   \nonumber\\ 
      & \quad + \Gamma +\frac{\sigma_w^2}{m}  + 2G_w^2 + 2\mu^2 D_{\mathcal{W}}^2 \Big).\label{eq: deviation 4} 
\end{align}
Re-arranging the terms in (\ref{eq: deviation 4}) and using the fact $1-18\eta^2\tau^2 L^2 \geq \frac{1}{2}$ yields
\begin{align}
   & \sum_{t=s\tau}^{(s+1)\tau}\mathbb{E}[\delta^{(t)}] \nonumber \\ & \quad \leq 18\eta^2\tau^2 \sum_{r=s\tau}^{(s+1)\tau}\Big(\sigma_w^2    + \Gamma +\frac{\sigma_w^2}{m} + 2 G_w^2 + 2\mu^2 D_{\mathcal{W}}^2  \Big).\nonumber
\end{align}
Summing over communication steps $s=0$ to $S-1$, and dividing both sides by $T=S\tau$ gives
\begin{align}
  & \frac{1}{T}\sum_{t=0}^{T} \mathbb{E}[\delta^{(t)}] \nonumber \\ & \quad \leq 18\eta^2\tau^2  \left(\sigma_w^2+ \Gamma+\frac{\sigma_w^2}{m}  + 2G_w^2 + 2\mu^2 D_{\mathcal{W}}^2  \right),\nonumber
\end{align}
as desired.
\subsection{PROOF OF LEMMA \ref{lemma: deviation2}}
\label{proof: deviation2}
Consider $s\tau \leq t\leq (s+1)\tau$. Recall that we only perform averaging based on a uniformly sampled subset of clients $\mathcal{D}^{(\lfloor{\frac{t}{\tau}}\rfloor)}$ chosen from $\mathcal{N}$. 
We use the following virtual updates, 
$\boldsymbol{w}_{i}^{(t)} = \prod_{\mathcal{W}}(\boldsymbol{{w}}^{({\lfloor{\frac{t}{\tau}}\rfloor}\tau)}- \eta \sum_{r = {{\lfloor{\frac{t}{\tau}}\rfloor}\tau}}^{t-1} \Big[  \nabla \ell_i(\bm{w} ^{(r)}_i;\xi_i^{(r)}) -  {\nabla}f_i(\boldsymbol{w}_i^{({\lfloor{\frac{t}{\tau}}\rfloor}\tau)}) - \mu (\boldsymbol{{w}}^{({\lfloor{\frac{t}{\tau}}\rfloor}\tau)} - \boldsymbol{w}_{i}^{(r)}) \Big ])$ and 
$\boldsymbol{w}^{(t)} =   \prod_{\mathcal{W}}(\boldsymbol{{w}}^{({\lfloor{\frac{t}{\tau}}\rfloor}\tau)}- \eta \sum_{r = {{\lfloor{\frac{t}{\tau}}\rfloor}\tau}}^{t-1} \bm{u}^{(r)})$. Following the updating rule, we have
\begin{align}
     &\mathbb{E}\Big[ \beta^{(t)} \Big] \nonumber \\  &\quad= \mathbb{E}\Big[\frac{1}{m}\sum_{i\in\mathcal{D}^{(\lfloor{\frac{t}{\tau}}\rfloor)}} \Big\|(\boldsymbol{w}_{i}^{(t)} - \boldsymbol{{w}}^{({\lfloor{\frac{t}{\tau}}\rfloor}\tau)})\Big\|^2\Big] \nonumber \\  &\quad \leq \mathbb{E}\Big[\frac{1}{m}\sum_{i\in\mathcal{D}^{(\lfloor{\frac{t}{\tau}}\rfloor)}} \Big\| \boldsymbol{{w}}^{({\lfloor{\frac{t}{\tau}}\rfloor}\tau)}-  \eta \sum_{r = {{\lfloor{\frac{t}{\tau}}\rfloor}\tau}}^{t-1} \Big[  \nabla \ell_i(\bm{w} ^{(r)}_i;\xi_i^{(r)}) \nonumber \\  &\quad \quad-  {\nabla}f_i(\boldsymbol{w}_i^{({\lfloor{\frac{t}{\tau}}\rfloor}\tau)}) - \mu (\boldsymbol{{w}}^{({\lfloor{\frac{t}{\tau}}\rfloor}\tau)} - \boldsymbol{w}_{i}^{(r)}) \Big ]  - \boldsymbol{{w}}^{({\lfloor{\frac{t}{\tau}}\rfloor}\tau)}\Big\|^2\Big]\nonumber \\ & \quad
     = \mathbb{E}\Big[\frac{1}{m}\sum_{i\in\mathcal{D}^{(\lfloor{\frac{t}{\tau}}\rfloor)}} \Big\| \eta \sum_{r = {{\lfloor{\frac{t}{\tau}}\rfloor}\tau}}^{t-1} \Big[  \nabla \ell_i(\bm{w} ^{(r)}_i;\xi_i^{(r)}) \nonumber \\ & \quad \quad- {\nabla}f_i(\boldsymbol{w}_i^{({\lfloor{\frac{t}{\tau}}\rfloor}\tau)}) - \mu (\boldsymbol{{w}}^{({\lfloor{\frac{t}{\tau}}\rfloor}\tau)} - \boldsymbol{w}_{i}^{(r)}) \Big ]\Big\|^2\Big]\nonumber \\ &\quad \leq \mathbb{E}\Big[\frac{1}{m} \sum_{i\in\mathcal{D}^{(\lfloor{\frac{t}{\tau}}\rfloor)}} \eta^2 \tau \sum_{r = {{\lfloor{\frac{t}{\tau}}\rfloor}\tau}}^{{({\lfloor{\frac{t}{\tau}}\rfloor} + 1)\tau}}  \Big\| \nabla \ell_i(\bm{w} ^{(r)}_i;\xi_i^{(r)}) \nonumber \\ & \quad \quad-  {\nabla}f_i(\boldsymbol{w}_i^{({\lfloor{\frac{t}{\tau}}\rfloor}\tau)}) - \mu (\boldsymbol{{w}}^{({\lfloor{\frac{t}{\tau}}\rfloor}\tau)} - \boldsymbol{w}_{i}^{(r)})\Big\|^2\Big]\nonumber \\ & \quad  
     = \eta^2 \tau \mathbb{E}\Big[\frac{1}{m} \sum_{i\in\mathcal{D}^{(\lfloor{\frac{t}{\tau}}\rfloor)}} \sum_{r = {{\lfloor{\frac{t}{\tau}}\rfloor}\tau}}^{{({\lfloor{\frac{t}{\tau}}\rfloor} + 1)\tau}}  \Big\| \nabla \ell_i(\bm{w} ^{(r)}_i;\xi_i^{(r)}) \nonumber \\ & \quad \quad- {\nabla}f_i(\boldsymbol{w}_i^{(r)}) + {\nabla}f_i(\boldsymbol{w}_i^{(r)}) - {\nabla}f_i(\boldsymbol{w}_i^{({\lfloor{\frac{t}{\tau}}\rfloor}\tau)}) \nonumber \\ & \quad \quad - \mu (\boldsymbol{{w}}^{({\lfloor{\frac{t}{\tau}}\rfloor}\tau)} - \boldsymbol{w}_{i}^{(r)})\Big\|^2\Big]. 
\end{align}
We apply Jensen's inequality to split the norms, which yields
\begin{align}
    \mathbb{E}[\beta^{(t)}]& \leq 3\eta^2\tau \sum_{r = {{\lfloor{\frac{t}{\tau}}\rfloor}\tau}}^{{({\lfloor{\frac{t}{\tau}}\rfloor} + 1)\tau}} \Big( \sigma_w^2 + 4 G_w^2 \nonumber \\ & \quad + \mu^2\mathbb{E}\Big[\frac{1}{m}\sum_{i\in\mathcal{D}^{(\lfloor{\frac{t}{\tau}}\rfloor)}} \Big\|(\boldsymbol{w}_{i}^{(r)} - \boldsymbol{{w}}^{({\lfloor{\frac{t}{\tau}}\rfloor}\tau)})\Big\|^2\Big] \Big)\nonumber \\ & \leq 3\eta^2\tau \sum_{r = {{\lfloor{\frac{t}{\tau}}\rfloor}\tau}}^{{({\lfloor{\frac{t}{\tau}}\rfloor} + 1)\tau}} \Big(\sigma_w^2 + 4 G_w^2 + \mu^2\mathbb{E} \Big[\beta^{(r)} \Big] \Big). \label{20}
\end{align}
We sum (\ref{20}) over $t = s\tau$ to $(s+1)\tau$ to get 
\begin{align}
&\sum_{t = {{\lfloor{\frac{t}{\tau}}\rfloor}\tau}}^{{({\lfloor{\frac{t}{\tau}}\rfloor} + 1)\tau}} \mathbb{E}[\beta^{(t)}] \nonumber \\ & \quad\leq 3\eta^2\tau \sum_{t = {{\lfloor{\frac{t}{\tau}}\rfloor}\tau}}^{{({\lfloor{\frac{t}{\tau}}\rfloor} + 1)\tau}} \sum_{r = {{\lfloor{\frac{t}{\tau}}\rfloor}\tau}}^{{({\lfloor{\frac{t}{\tau}}\rfloor} + 1)\tau}} \Big(\sigma_w^2 + 4 G_w^2 + \mu^2\mathbb{E} \left[\beta^{(r)} \right] \Big)\nonumber \\ & \quad \leq 3\eta^2\tau^2 \sum_{r = {{\lfloor{\frac{t}{\tau}}\rfloor}\tau}}^{{({\lfloor{\frac{t}{\tau}}\rfloor} + 1)\tau}} \Big(\sigma_w^2 + 4 G_w^2 + \mu^2\mathbb{E} \left[\beta^{(r)} \right] \Big).
\end{align}
We rearrange the terms and use the fact that $1 - 3\eta^2 \tau^2 \mu^2 \geq \frac{1}{2}$ yields
\begin{align}
\sum_{t = {{\lfloor{\frac{t}{\tau}}\rfloor}\tau}}^{{({\lfloor{\frac{t}{\tau}}\rfloor} + 1)\tau}} \mathbb{E}[\beta^{(t)}]& \leq 6\eta^2\tau^2 \sum_{r = {{\lfloor{\frac{t}{\tau}}\rfloor}\tau}}^{{({\lfloor{\frac{t}{\tau}}\rfloor} + 1)\tau}} \Big(\sigma_w^2 + 4 G_w^2 \Big).
\end{align}
Summing over the the communication steps $s = 0$ to $S-1$, then dividing both sides by $T = S\tau$ gives 
\begin{align}
\frac{1}{T}\sum_{t = 0}^{T} \mathbb{E}[\beta^{(t)}]& \leq 6\eta^2\tau^2  \Big(\sigma_w^2 + 4 G_w^2 \Big).
\end{align}
\subsection{PROOF OF LEMMA \ref{lemma: deviation3}}
\label{proof: deviation3}
Consider $s\tau \leq t\leq (s+1)\tau$. Recall that we only perform averaging based on a uniformly sampled subset of clients $\mathcal{D}^{(\lfloor{\frac{t}{\tau}}\rfloor)}$ chosen from $\mathcal{N}$. 
Following the updating rule, we have
\begin{align}
    &\mathbb{E}\Big[ \varphi^{(t)} \Big] \nonumber \\ & \,\,= \mathbb{E} \Big[ \frac{1}{m} \sum_{i\in{\mathcal{D}^{(\lfloor{\frac{t}{\tau}}\rfloor)}}} \| \bm{w}_i^{t+1} - \bm{w}^t \|^2 \Big]\nonumber \\ & \,\, \leq \mathbb{E} \Big[ \frac{1}{m} \sum_{i\in{\mathcal{D}^{(\lfloor{\frac{t}{\tau}}\rfloor)}}} \| \bm{w}_i^{t} - \eta ( \nabla \ell_i( \bm{w} ^{(t)}_i;\xi^{(t)}_i)-  {\nabla}f_i(\boldsymbol{w}_i^{({\lfloor{\frac{t}{\tau}}\rfloor}\tau)}) \nonumber \\ & \quad \quad - \mu (\boldsymbol{{w}}^{({\lfloor{\frac{t}{\tau}}\rfloor}\tau)} - \boldsymbol{w}_{i}^{(t)})) - \bm{w}^t \|^2 \Big]\nonumber \\& \,\, \leq \mathbb{E} \Big[ \frac{1}{m} \sum_{i\in{\mathcal{D}^{(\lfloor{\frac{t}{\tau}}\rfloor)}}} \| \bm{w}_i^{t} - \bm{w}^t - \eta ( \nabla \ell_i( \bm{w} ^{(t)}_i;\xi^{(t)}_i) + {\nabla}f_i(\boldsymbol{w}_i^{t}) \nonumber \\& \quad  \quad - {\nabla}f_i(\boldsymbol{w}_i^{({\lfloor{\frac{t}{\tau}}\rfloor}\tau)}) -{\nabla}f_i(\boldsymbol{w}_i^{t}) - \mu (\boldsymbol{{w}}^{({\lfloor{\frac{t}{\tau}}\rfloor}\tau)} - \boldsymbol{w}_{i}^{(t)})) \|^2 \Big].
\end{align}
We apply Jensen's inequality to split the norms, which yields
\begin{align}
    \mathbb{E}[\varphi^{(t)}]&  \leq 4  
\Big[ \mathbb{E} \Big[ \frac{1}{m} \sum_{i\in{\mathcal{D}^{(\lfloor{\frac{t}{\tau}}\rfloor)}}} \| \bm{w}_i^{t} - \bm{w}^t \|^2 \Big] \nonumber \\& \quad+ \eta^2 \mathbb{E} \Big[\frac{1}{m} \sum_{i\in{\mathcal{D}^{(\lfloor{\frac{t}{\tau}}\rfloor)}}} \| \nabla \ell_i( \bm{w} ^{(t)}_i;\xi^{(t)}_i) + {\nabla}f_i(\boldsymbol{w}_i^{t})  \|^2 \Big] \nonumber \\& \quad + \eta^2\mathbb{E} \Big[ \frac{1}{m} \sum_{i\in{\mathcal{D}^{(\lfloor{\frac{t}{\tau}}\rfloor)}}}\|{\nabla}f_i(\boldsymbol{w}_i^{({\lfloor{\frac{t}{\tau}}\rfloor}\tau)})      -{\nabla}f_i(\boldsymbol{w}_i^{t})\|^2 \Big] \nonumber \\& \quad + \eta^2\mu^2 \mathbb{E} \Big[ \frac{1}{m} \sum_{i\in{\mathcal{D}^{(\lfloor{\frac{t}{\tau}}\rfloor)}}}\|\boldsymbol{{w}}^{({\lfloor{\frac{t}{\tau}}\rfloor}\tau)} - \boldsymbol{w}_{i}^{(t)} \|^2 \Big] \Big] \nonumber \\ &  \leq 4  
\Big[ \mathbb{E} \Big[ \delta^{(t)} \Big] + \eta^2 \frac{\sigma_w^2}{m}    + 4 \eta^2G_w^2 + \eta^2\mu^2 \mathbb{E} \Big[ \beta^{(t)} \Big] \Big]. \label{24}
\end{align}
We sum (\ref{24}) over $t = s\tau$ to $(s+1)\tau$ to get
\begin{align}
&\sum_{t = {{\lfloor{\frac{t}{\tau}}\rfloor}\tau}}^{{({\lfloor{\frac{t}{\tau}}\rfloor} + 1)\tau}} \mathbb{E}[\varphi^{(t)}] \nonumber \\ & \,\,\,\, \leq 4\sum_{t = {{\lfloor{\frac{t}{\tau}}\rfloor}\tau}}^{{({\lfloor{\frac{t}{\tau}}\rfloor} + 1)\tau}}   
\left[ \mathbb{E} \Big[ \delta^{(t)} \Big] + \eta^2 \frac{\sigma_w^2}{m}    + 4 \eta^2G_w^2 + \eta^2\mu^2 \mathbb{E} \Big[ \beta^{(t)} \Big] \right].
\end{align}
Summing over the communication steps $s = 0$ to $S-1$, then dividing both sides by $T = S\tau$ yields
\begin{align}
\frac{1}{T}\sum_{t = 0}^{T} \mathbb{E}[\varphi^{(t)}]& \leq  4\frac{1}{T}\sum_{t = 0}^{T} \mathbb{E}[\delta^{(t)}] + 4\eta^2 \frac{\sigma_w^2}{m}  \nonumber \\ & \quad+ 16 \eta^2G_w^2 + \eta^2\mu^2 \frac{1}{T}\sum_{t = 0}^{T} \mathbb{E}[\beta^{(t)}] \nonumber \\ & 
  \leq 72\eta^2\tau^2  \left(\sigma_w^2+\frac{\sigma_w^2}{m} + \Gamma + 2 G_w^2 + 2\mu^2 D_{\mathcal{W}}^2  \right) \nonumber \\ & \quad+  4\eta^2 \frac{\sigma_w^2}{m}  + 16 \eta^2 G_w^2 \nonumber \\ & \quad+ 6\eta^4\tau^2  \mu^2\Big(\sigma_w^2 + 4 G_w^2 \Big).
\end{align}
\subsection{PROOF OF LEMMA \ref{lemma: one iteration w}}\label{proof: one iteration w}
From the global model updating rule, we have
\begin{align}
\bm{w}^{t+1} - \bm{w}^t &= \frac{1}{m} \sum_{i\in \mathcal{D}^{(\lfloor{\frac{t}{\tau}}\rfloor)}} (\bm{w}_i^{t+1}-\bm{w}_i^{t}) - \frac{1}{\mu}(\bm{c}^{t+1} - \bm{c}^t).
\end{align}
Substituting the expression for the control variable update
\begin{align}
\bm{c}^{t+1} - \bm{c}^t &= -  \frac{\mu}{N} \sum_{i\in \mathcal{D}^{(\lfloor{\frac{t}{\tau}}\rfloor)}} (\bm{w}_i^{t+1} - \bm{w}^t).
\end{align}
Replacing this into the global update rule gives
\begin{align}
&\bm{w}^{t+1} - \bm{w}^t \nonumber \\ & \quad = \frac{1}{m} \sum_{i\in \mathcal{D}^{(\lfloor{\frac{t}{\tau}}\rfloor)}} (\bm{w}_i^{t+1} - \bm{w}_i^{t}) + \frac{1}{N} \sum_{i\in \mathcal{D}^{(\lfloor{\frac{t}{\tau}}\rfloor)}}(\bm{w}_i^{t+1} - \bm{w}^t).
\end{align}
Hence, we obtain the expression for the global model update 
\begin{align}
  \bm{w}^{t+1}&=  \prod_{\mathcal{W}}\Big( \bm{w}^t -\frac{1}{m} \eta\sum_{i\in{\mathcal{D}^{(\lfloor{\frac{t}{\tau}}\rfloor)}}} ( \nabla \ell_i( \bm{w} ^{(t)}_i;\xi^{(t)}_i)\nonumber \\ & \quad-  {\nabla}f_i(\boldsymbol{w}_i^{({\lfloor{\frac{t}{\tau}}\rfloor}\tau)}) - \mu (\boldsymbol{{w}}^{({\lfloor{\frac{t}{\tau}}\rfloor}\tau)} - \boldsymbol{w}_{i}^{(t)}))\Big) \nonumber\\
  & \quad + \frac{1}{N} \sum_{i\in \mathcal{D}^{(\lfloor{\frac{t}{\tau}}\rfloor)}}(\bm{w}_i^{t+1} - \bm{w}^t). 
  \end{align}
The one iteration primal update is expressed as follows
  \begin{align}
   &\mathbb{E}\|\bm{w}^{(t+1)} - \bm{w} \|^2 \nonumber\\
  & \quad =  \mathbb{E} \Big[ \Big.\| \prod_{\mathcal{W}} \Big(\bm{w}^t -\frac{1}{m} \eta\sum_{i\in{\mathcal{D}^{(\lfloor{\frac{t}{\tau}}\rfloor)}}}  \big(\nabla \ell_i( \bm{w} ^{(t)}_i;\xi^{(t)}_i)\nonumber\\
  & \quad \quad-  {\nabla}f_i(\boldsymbol{w}_i^{({\lfloor{\frac{t}{\tau}}\rfloor}\tau)}) - \mu (\boldsymbol{{w}}^{({\lfloor{\frac{t}{\tau}}\rfloor}\tau)} - \boldsymbol{w}_{i}^{(t)})\big)\Big)   
   \nonumber \\   
   &\quad \quad+ \frac{1}{N} \sum_{i\in \mathcal{D}^{(\lfloor{\frac{t}{\tau}}\rfloor)}}(\bm{w}_i^{t+1} - \bm{w}^t) -\bm{w} \|^2 \Big]   \nonumber\\
   &\quad=  \mathbb{E} \Big[ \Big.\| \prod_{\mathcal{W}} \Big(\bm{w}^t -\frac{1}{m} \eta\sum_{i\in{\mathcal{D}^{(\lfloor{\frac{t}{\tau}}\rfloor)}}} \big(\nabla \ell_i( \bm{w} ^{(t)}_i;\xi^{(t)}_i) \nonumber \\   
   &\quad \quad - \nabla f_i( \bm{w} ^{(t)}_i)  + \nabla f_i( \bm{w} ^{(t)}_i)  - {\nabla}f_i(\boldsymbol{w}_i^{({\lfloor{\frac{t}{\tau}}\rfloor}\tau)})  \Big.
\nonumber \\   
   &\quad \quad  -\mu (\boldsymbol{{w}}^{({\lfloor{\frac{t}{\tau}}\rfloor}\tau)} - \boldsymbol{w}_{i}^{(t)})\big)\Big)  
   \nonumber \\   
   &\quad \quad  + \frac{1}{N} \sum_{i\in \mathcal{D}^{(\lfloor{\frac{t}{\tau}}\rfloor)}}(\bm{w}_i^{t+1} - \bm{w}^t) -\bm{w} \|^2 \Big]  \nonumber\\
    & \quad  \leq \mathbb{E}\|\bm{w} ^{(t)}  -\frac{1}{m} \eta\sum_{i\in{\mathcal{D}^{(\lfloor{\frac{t}{\tau}}\rfloor)}}}   \nabla f_i( \bm{w} ^{(t)}_i) - \bm{w}\|^2 +  \mathbb{E}\| \nonumber \\   
   &\quad \quad -\frac{1}{m} \eta\sum_{i\in{\mathcal{D}^{(\lfloor{\frac{t}{\tau}}\rfloor)}}} \big(\nabla \ell_i( \bm{w} ^{(t)}_i;\xi^{(t)}_i) -  \nabla f_i( \bm{w} ^{(t)}_i)  \nonumber \\   
   &\quad \quad-  \nabla f_i(\boldsymbol{w}_i^{({\lfloor{\frac{t}{\tau}}\rfloor}\tau)})  - \mu (\boldsymbol{{w}}^{({\lfloor{\frac{t}{\tau}}\rfloor}\tau)} - \boldsymbol{w}_{i}^{(t)})\big)   
   \nonumber \\   
   &\quad \quad  + \frac{1}{N} \sum_{i\in \mathcal{D}^{(\lfloor{\frac{t}{\tau}}\rfloor)}}(\bm{w}_i^{t+1} - \bm{w}^t)    \|^2 \nonumber\\
   & \quad  \leq \mathbb{E}\|\bm{w} ^{(t)}- \bm{w}\|^2 \nonumber\\ & \quad \quad +  \underbrace{\mathbb{E}[- 2 \eta \langle \frac{1}{m}\sum_{i\in \mathcal{D}^{(\lfloor{\frac{t}{\tau}}\rfloor)}} \nabla f_i (\bm{w} ^{(t)}_i),\bm{w} ^{(t)}- \bm{w} \rangle]}_{T_1}\nonumber\\ & \quad \quad + \underbrace{ \eta^2 \mathbb{E}\| \frac{1}{m}\sum_{i\in \mathcal{D}^{(\lfloor{\frac{t}{\tau}}\rfloor)}} \nabla f_i \bm{w} ^{(t)}_i)\|^2}_{T_2} \nonumber\\ & \quad \quad  + \underbrace{\mathbb{E}\|-\frac{\eta}{m} \sum_{i\in{\mathcal{D}^{(\lfloor{\frac{t}{\tau}}\rfloor)}}} \big(\nabla \ell_i( \bm{w} ^{(t)}_i;\xi^{(t)}_i)  
-  \nabla f_i( \bm{w} ^{(t)}_i)}_{T_3} \nonumber\\ &\quad \quad -  \underbrace{ \nabla f_i(\boldsymbol{w}_i^{({\lfloor{\frac{t}{\tau}}\rfloor}\tau)}) - \mu (\boldsymbol{{w}}^{({\lfloor{\frac{t}{\tau}}\rfloor}\tau)} - \boldsymbol{w}_{i}^{(t)}) \big)}_{T_3} \nonumber\\ &\quad \quad + \underbrace{\frac{1}{N} \sum_{i\in \mathcal{D}^{(\lfloor{\frac{t}{\tau}}\rfloor)}}(\bm{w}_i^{t+1} - \bm{w}^t) \|^2}_{T_3}.  \label{l20}
\end{align}
We start by bounding $T_1$ first as follows
\begin{align}
   T_1 &=  \mathbb{E}_{\mathcal{D}^{(\lfloor{\frac{t}{\tau}}\rfloor)}}\Big[ \frac{1}{m} \sum_{i\in \mathcal{D}^{(\lfloor{\frac{t}{\tau}}\rfloor)}}\Big[- 2\eta\Big \langle\nabla f_i(\bm{w}^{(t)}_i), \bm{w} ^{(t)}- \bm{w}^{(t)}_i  \Big\rangle \nonumber\\ &\quad - 2\eta \Big\langle \nabla f_i(\bm{w} ^{(t)}_i), \bm{w}^{(t)}_i - \bm{w}  \Big\rangle\Big]\Big]  \label{l21} \\
   &\leq \mathbb{E}_{\mathcal{D}^{(\lfloor{\frac{t}{\tau}}\rfloor)}}\Big[2\eta \frac{1}{m}\sum_{i\in \mathcal{D}^{(\lfloor{\frac{t}{\tau}}\rfloor)}}  \Big[f_i(\bm{w} ^{(t)}_i) - f_i(\bm{w} ^{(t)})\nonumber\\ &\quad + \frac{L}{2}\|\bm{w} ^{(t)} - \bm{w}^{(t)}_i \|^2 +  f_i(\bm{w}) - f_i(\bm{w} ^{(t)}_i) \Big]\Big] \label{l22} \\
   & = -2\eta \mathbb{E}\left[\sum_{i=1}^N  \lambda^{(\lfloor{\frac{t}{\tau}}\rfloor)}_i f_i(\bm{w} ^{(t)})  -  \lambda^{(\lfloor{\frac{t}{\tau}}\rfloor)}_i f_i(\bm{w})\right] \nonumber\\ &\quad + L\eta \mathbb{E}\Big[\frac{1}{m}\sum_{i\in \mathcal{D}^{(\lfloor{\frac{t}{\tau}}\rfloor)}} \|\bm{w} ^{(t)} - \bm{w}^{(t)}_i \|^2 \Big]\nonumber\\
   & = -2\eta\mathbb{E}\Big[F(\bm{w} ^{(t)},\boldsymbol{\lambda}^{(\lfloor{\frac{t}{\tau}}\rfloor)})  - F(\bm{w} ,\boldsymbol{\lambda}^{(\lfloor{\frac{t}{\tau}}\rfloor)})\Big] + L\eta \mathbb{E}\Big[ \delta^{(t)} \Big],\nonumber
\end{align}
where from~(\ref{l21})  to~(\ref{l22}) we use the smoothness and convexity properties.\\
Next, the term $T_2$ is bounded as follows
\begin{align}
   T_2 &= \eta^2  \mathbb{E}\left\| \frac{1}{m}\sum_{i\in\mathcal{D}^{(\lfloor{\frac{t}{\tau}}\rfloor)}} \nabla  f_i(\bm{w} ^{(t)}_i) \right\|^2\nonumber\\ & \leq \eta^2 \frac{1}{m}\sum_{i\in\mathcal{D}^{(\lfloor{\frac{t}{\tau}}\rfloor)}} \mathbb{E}\left\|   \nabla  f_i(\bm{w} ^{(t)}_i) \right\|^2    \leq \eta^2 G_w^2. 
\end{align} 
Finally, we bound the term $T_3$ as follows
\begin{align}
   T_3 &= \mathbb{E}\|-\frac{1}{m} \eta\sum_{i\in{\mathcal{D}^{(\lfloor{\frac{t}{\tau}}\rfloor)}}} \big(\nabla \ell_i( \bm{w} ^{(t)}_i;\xi^{(t)}_i) -  \nabla f_i( \bm{w} ^{(t)}_i) \nonumber\\ &\quad -  \nabla f_i(\boldsymbol{w}_i^{({\lfloor{\frac{t}{\tau}}\rfloor}\tau)}) - \mu (\boldsymbol{{w}}^{({\lfloor{\frac{t}{\tau}}\rfloor}\tau)} - \boldsymbol{w}_{i}^{(t)}) \big) \nonumber\\ & \quad + \frac{1}{N} \sum_{i\in \mathcal{D}^{(\lfloor{\frac{t}{\tau}}\rfloor)}}(\bm{w}_i^{t+1} - \bm{w}^t) \|^2 \nonumber  \\ & \leq 4 
\Big[ \mathbb{E}  \| \frac{1}{m} \eta\sum_{i\in{\mathcal{D}^{(\lfloor{\frac{t}{\tau}}\rfloor)}}} \big(\nabla f_i( \bm{w} ^{(t)}_i) - \nabla \ell_i( \bm{w} ^{(t)}_i;\xi^{(t)}_i)   \big) \|^2 \nonumber\\ &\quad + \mathbb{E}  \| \frac{1}{m} \eta\sum_{i\in{\mathcal{D}^{(\lfloor{\frac{t}{\tau}}\rfloor)}}}\nabla f_i(\boldsymbol{w}_i^{({\lfloor{\frac{t}{\tau}}\rfloor}\tau)}) \|^2 \nonumber \\
& \quad + \mathbb{E}  \| \frac{1}{m} \eta\sum_{i\in{\mathcal{D}^{(\lfloor{\frac{t}{\tau}}\rfloor)}}}\mu (\boldsymbol{{w}}^{({\lfloor{\frac{t}{\tau}}\rfloor}\tau)} - \bm{w}_i^t) \|^2 \nonumber\\ &\quad+ \mathbb{E}  \| \frac{1}{N} \sum_{i\in \mathcal{D}^{(\lfloor{\frac{t}{\tau}}\rfloor)}}(\bm{w}_i^{t+1} - \bm{w}^t) \|^2
\Big] \nonumber \\
& \leq 4 \eta^2 \mathbb{E}\left[\|\bm{u}^{(t)} - \Bar{\bm{u}}^{(t)}\|^2\right] + \eta^2 G_w^2 \nonumber \\
& \quad+ 4\eta^2 \mu^2 \mathbb{E}  \| \frac{1}{m} \sum_{i\in{\mathcal{D}^{(\lfloor{\frac{t}{\tau}}\rfloor)}}}(\boldsymbol{{w}}^{({\lfloor{\frac{t}{\tau}}\rfloor}\tau)} - \bm{w}_i^t) \|^2  \nonumber \\
& \quad+ 4\mathbb{E}  \| \frac{1}{N} \sum_{i\in \mathcal{D}^{(\lfloor{\frac{t}{\tau}}\rfloor)}}(\bm{w}_i^{t+1} - \bm{w}^t) \|^2 \nonumber \\
& \leq 4 \eta^2 \mathbb{E}\left[\|\bm{u}^{(t)} - \Bar{\bm{u}}^{(t)}\|^2\right] + \eta^2 G_w^2 \nonumber \\
& \quad + 4\eta^2 \mu^2  \mathbb{E} \Big[ \frac{1}{m} \sum_{i\in{\mathcal{D}^{(\lfloor{\frac{t}{\tau}}\rfloor)}}}\|\boldsymbol{{w}}^{({\lfloor{\frac{t}{\tau}}\rfloor}\tau)} - \bm{w}_i^t \|^2 \Big]  \nonumber \\
& \quad + 4 \frac{m^2}{N^2} \mathbb{E} \Big[ \frac{1}{m} \sum_{i\in{\mathcal{D}^{(\lfloor{\frac{t}{\tau}}\rfloor)}}} \| \bm{w}_i^{t+1} - \bm{w}^t \|^2 \Big] \nonumber \\
& \leq 4 \eta^2 \mathbb{E}\left[\|\bm{u}^{(t)} - \Bar{\bm{u}}^{(t)}\|^2\right] + \eta^2 G_w^2\nonumber \\
& \quad + 4\eta^2 \mu^2 \mathbb{E} \Big[\beta^{(t)} \Big] + 4\mathbb{E} \Big[\varphi^{(t)} \Big]
\end{align} 
Plugging ($T_1$), ($T_2$), and ($T_3$) back in (\ref{l20}) gives
\begin{align}
&\mathbb{E}\|\bm{w}^{(t+1)} - \bm{w} \|^2 \nonumber \\
&\quad \leq \mathbb{E}\|\bm{w}^{(t)} - \bm{w} \|^2 \nonumber \\
& \quad\quad -2\eta\mathbb{E}\left[F(\bm{w} ^{(t)},\boldsymbol{\lambda}^{(\lfloor{\frac{t}{\tau}}\rfloor)})  - F(\bm{w} ,\boldsymbol{\lambda}^{(\lfloor{\frac{t}{\tau}}\rfloor)})\right]\nonumber \\
& \quad \quad+ L\eta \mathbb{E}\Big[ \delta^{(t)} \Big] + 2\eta^2 G_w^2 + 4 \eta^2 \mathbb{E}\left[\|\bm{u}^{(t)} - \Bar{\bm{u}}^{(t)}\|^2\right] \nonumber \\
& \quad \quad+ 4\eta^2 \mu^2 \mathbb{E} \Big[\beta^{(t)} \Big] + 4\mathbb{E} \Big[\varphi^{(t)} \Big].
\end{align}
which concludes the proof.
\subsection{PROOF OF LEMMA \ref{lemma: one iteration lambda}}\label{proof: one iteration lambda}
According to the updating rule for $\boldsymbol{\lambda}$ and the fact that $F$ is linear in $\boldsymbol{\lambda}$, we have
\begin{equation}
    \begin{aligned}
     &\mathbb{E}\Big\|\boldsymbol{\lambda}^{(s+1)} - \boldsymbol{\lambda}\Big\|^2 \nonumber\\& \quad= \mathbb{E}\Big\|\prod_{\Lambda}(\boldsymbol{\lambda}^{(s)}+ \Delta_{s})- \boldsymbol{\lambda} \Big\|^2\nonumber\\& \quad\leq \mathbb{E}\Big\|\boldsymbol{\lambda}^{(s)}- \boldsymbol{\lambda} + \Delta_{s}\Big\|^2\nonumber\\& \quad= \mathbb{E}\Big\|\boldsymbol{\lambda}^{(s)}- \boldsymbol{\lambda} + \bar{\Delta}_{s}\Big\|^2 + \mathbb{E}\Big\|\Delta_{s} - \bar{\Delta}_{s}\Big\|^2\nonumber\\& \quad = \mathbb{E}\|\boldsymbol{\lambda}^{(s)}- \boldsymbol{\lambda} \|^2 + \mathbb{E}\Big[2\Big\langle \bar{\Delta}_{s},  \boldsymbol{\lambda}^{(s)}- \boldsymbol{\lambda}\Big\rangle\Big] + \mathbb{E}\|\bar{\Delta}_{s}\|^2\nonumber\\& \quad \quad +\mathbb{E} \|\Delta_{s} - \bar{\Delta}_{s}\|^2\nonumber\\& \quad  = \mathbb{E}\|\boldsymbol{\lambda}^{(s)}- \boldsymbol{\lambda} \|^2\nonumber\\& \quad  \quad + 2\gamma \sum_{t= s\tau+1}^{(s+1)\tau}\mathbb{E}\Big[\Big\langle \nabla_{\boldsymbol{\lambda}} F(\bm{w}^{(t)},\boldsymbol{\lambda}^{(s)}),  \boldsymbol{\lambda}^{(s)}- \boldsymbol{\lambda}\Big\rangle\Big] \nonumber\\& \quad \quad + \mathbb{E}\|\bar{\Delta}_{s}\|^2 + \mathbb{E}\|\Delta_{s} - \bar{\Delta}_{s}\|^2\nonumber\\& \quad=\|\boldsymbol{\lambda}^{(s)}- \boldsymbol{\lambda}\|^2 \nonumber\\& \quad \quad -2\gamma \sum_{t= s\tau+1}^{(s+1)\tau} \mathbb{E}\Big[F(\bm{w} ^{(t)},\boldsymbol{\lambda})-F(\bm{w} ^{(t)},\boldsymbol{\lambda}^{(s)}))\Big]\nonumber\\& \quad \quad + \mathbb{E}\|\bar{\Delta}_{s}\|^2 +  \mathbb{E}\|\Delta_{s} - \bar{\Delta}_{s}\|^2 \label{l41}.
\end{aligned}
\end{equation}

\subsection{PROOF OF THEOREM~\ref{theorem1}}\label{proof: theorem1}
Equipped with the above results, we now turn to proving Theorem~\ref{theorem1}. We start by noting that $\forall \bm{w} \in \mathcal{W}$, $\forall \boldsymbol{\lambda} \in \Lambda$,  according the convexity of global objective w.r.t. $\bm{w} $ and its linearity in terms of $\boldsymbol{\lambda}$, we have
\begin{align}
     &\mathbb{E}[F(\hat{\bm{w} },\boldsymbol{\lambda} ) - \mathbb{E}[F(\bm{w}  ,\hat{\boldsymbol{\lambda}})]\nonumber\\
     & \quad \leq \frac{1}{T}\sum_{t=1}^T \Big\{   \mathbb{E}\Big[F( \bm{w}^{(t)},\boldsymbol{\lambda} )\Big] - \mathbb{E}\Big[F(\bm{w}  ,\boldsymbol{\lambda}^{(\lfloor{\frac{t}{\tau}}\rfloor)})\Big] \Big\}\nonumber\\
     & \quad \leq \frac{1}{T}\sum_{t=1}^T\Big \{   \mathbb{E}\Big[F( \bm{w}^{(t)},\boldsymbol{\lambda} )\Big] -\mathbb{E}\Big[F( \bm{w}^{(t)},\boldsymbol{\lambda}^{(\lfloor{\frac{t}{\tau}}\rfloor)})\Big] \nonumber\\& \quad \quad+\mathbb{E}\Big[ F( \bm{w}^{(t)},\boldsymbol{\lambda}^{(\lfloor{\frac{t}{\tau}}\rfloor)})\Big] - \mathbb{E}\Big[F(\bm{w},\boldsymbol{\lambda}^{(\lfloor{\frac{t}{\tau}}\rfloor)})\Big] \Big\}\nonumber\\
     & \quad \leq \frac{1}{T}\sum_{s=0}^{S-1} \sum_{t=s\tau+1}^{(s+1)\tau}\mathbb{E}\{F( \bm{w}^{(t)},\boldsymbol{\lambda} ) -F( \bm{w}^{(t)},\boldsymbol{\lambda}^{(s)})\}\label{eq: thm1 1}  \\ 
     & \quad \quad +\frac{1}{T}\sum_{t=1}^T\mathbb{E}\{ F( \bm{w}^{(t)},\boldsymbol{\lambda}^{(\lfloor{\frac{t}{\tau}}\rfloor)}) - F(\bm{w}  ,\boldsymbol{\lambda}^{(\lfloor{\frac{t}{\tau}}\rfloor)}) \}  \label{eq: thm1 2},
\end{align}
To bound the term in (\ref{eq: thm1 1}), plugging Lemma~\ref{lemma: bounded variance of lambda}  into Lemma~\ref{lemma: one iteration lambda}, we have
\begin{align}
   & \frac{1}{T}\sum_{s=0}^{S-1} \sum_{t=s\tau+1}^{(s+1)\tau} \mathbb{E}(F(\bm{w} ^{(t)},\boldsymbol{\lambda})-F(\bm{w} ^{(t)},\boldsymbol{\lambda}^{(s)})) \nonumber\\& \quad \leq  \frac{1}{2\gamma T}\|\boldsymbol{\lambda}^{(0)}- \boldsymbol{\lambda}\|^2 + \frac{\gamma\tau }{2}G_{\lambda}^2+ \frac{\gamma\tau\sigma_{ \lambda}^2 }{2m} \nonumber\\
    & \quad \leq \frac{D_{\Lambda}^2}{2\gamma T}  + \frac{\gamma \tau G_{\boldsymbol{\lambda}}^2}{2}  + \frac{\gamma\tau\sigma_{ \lambda}^2 }{2m}.
\end{align}
To bound the term in (\ref{eq: thm1 2}), we plug Lemmas~\ref{lemma: bounded variance of w}, ~\ref{lemma: deviation1}, ~\ref{lemma: deviation2}, and ~\ref{lemma: deviation3} into Lemma~\ref{lemma: one iteration w} and apply the telescoping sum from $t = 1$ to $T$ to get
\begin{align}
    &\frac{1}{T}\sum_{t=1}^T \mathbb{E}(F(\bm{w} ^{(t)},\boldsymbol{\lambda}^{(\lfloor{\frac{t}{\tau}}\rfloor)})-F(\bm{w}  ,\boldsymbol{\lambda}^{(\lfloor{\frac{t}{\tau}}\rfloor)})) \nonumber\\
    & \quad  \leq \frac{1}{2T\eta} \mathbb{E}\|\bm{w}^{(0)} - \bm{w} \|^2 + \frac{L}{2} \frac{1}{T} \sum_{t=1}^T \mathbb{E}\Big[ \delta^{(t)} \Big] + \eta G_w^2  \nonumber\\
    & \quad \quad + \frac{2 \eta \sigma^2_w}{m} + 2\eta \mu^2 \frac{1}{T} \sum_{t=1}^T \mathbb{E} \Big[\beta^{(t)} \Big] + \frac{2}{\eta}\frac{1}{T} \sum_{t=1}^T \mathbb{E} \Big[\varphi^{(t)} \Big]\nonumber \\ & \quad \leq \frac{D_{\mathcal{W}}^2}{2T\eta} +  9 L\eta^2\tau^2  \left(\sigma_w^2+\frac{\sigma_w^2}{m} + \Gamma + 2G_w^2 + 2\mu^2 D_{\mathcal{W}}^2 \right)  \nonumber\\
    & \quad \quad + \eta G_w^2 + 10\frac{\eta \sigma^2_w}{m} + 24   \eta^3\tau^2  \mu^2  \Big(\sigma_w^2 + 4 G_w^2 \Big) + 32 \eta G_w^2  \nonumber\\
    & \quad \quad+ 144\eta\tau^2  \left(\sigma_w^2+\frac{\sigma_w^2}{m} + \Gamma + 2G_w^2 + 2\mu^2 D_{\mathcal{W}}^2  \right).
\end{align}
Putting pieces together, and taking maximum over dual $\boldsymbol{\lambda}$, minimum over primal $\bm{w}$ yields
\begin{align}
     &\min_{\bm{w}\in \mathcal{W}}\max_{\boldsymbol{\lambda}\in \Lambda} \mathbb{E}[F(\hat{\bm{w} },\boldsymbol{\lambda} ) - \mathbb{E}[F(\bm{w}  ,\hat{\boldsymbol{\lambda}})]\nonumber\\
     & \quad \leq \frac{D_{\mathcal{W}}^2}{2T\eta} + 9 L\eta^2\tau^2  \left(\sigma_w^2+\frac{\sigma_w^2}{m} + \Gamma + 2 G_w^2 + 2\mu^2 D_{\mathcal{W}}^2 \right) \nonumber \\ & \quad \quad+ \eta G_w^2 + 10\frac{\eta \sigma^2_w}{m} + 24   \eta^3\tau^2  \mu^2  \Big(\sigma_w^2 + 4 G_w^2 \Big)  + 32\eta G_w^2 \nonumber \\ & \quad \quad  + 144\eta\tau^2  \left(\sigma_w^2+\frac{\sigma_w^2}{m} + \Gamma + 2G_w^2 + 2\mu^2 D_{\mathcal{W}}^2 \right)\nonumber \\ & \quad \quad + \frac{D_{\Lambda}^2}{2\gamma T}  + \frac{\gamma \tau G_{\boldsymbol{\lambda}}^2}{2}  + \frac{\gamma\tau\sigma_{ \lambda}^2 }{2m}.
\end{align} 
Plugging in $\tau =  \frac{T^{1/4}}{\sqrt{m}}$, $\eta = \frac{1}{4L \sqrt{T}}$, and $\gamma = \frac{1}{T^{5/8}}$, and $\mu = 2L \sqrt{\frac{N}{m}}$,  we  conclude the proof by getting
\begin{align}
    &\max_{\boldsymbol{\lambda}\in \Lambda}\mathbb{E}[F(\hat{\boldsymbol{w}},\boldsymbol{\lambda} )] -\min_{\bm{w}\in\mathcal{W}} \mathbb{E}[F(\boldsymbol{w} ,\hat{\boldsymbol{\lambda}} )] 
    \nonumber \\&\quad  \leq  \frac{9}{16L}  \frac{1}{\sqrt{T}m}  \left(\sigma_w^2+\frac{\sigma_w^2}{m} + \Gamma + 2G_w^2 + 2\mu^2 D_{\mathcal{W}}^2  \right)   \nonumber \\&\quad \quad+  2L\frac{D_{\mathcal{W}}^2}{\sqrt{T}} + \frac{1}{4L}\frac{ G_w^2}{\sqrt{T}} + \frac{5}{2L}\frac{\sigma^2_w}{\sqrt{T}m} \nonumber \\ & \quad \quad+ \frac{3N}{2L} \frac{1}{Tm} \Big(\sigma_w^2 + 4 G_w^2 \Big) \nonumber \\ & \quad \quad +  \frac{36}{m}  \left(\sigma_w^2+\frac{\sigma_w^2}{m} + \Gamma  + 2G_w^2 + 2\mu^2 D_{\mathcal{W}}^2  \right) + \frac{8}{L} \frac{G_w^2}{\sqrt{T}}  \nonumber \\ & \quad \quad+ \frac{1}{2}\frac{D_{\Lambda}^2}{T^{3/8}}  + \frac{1}{2}\frac{G_{\boldsymbol{\lambda}}^2}{{m}^{1/2}T^{3/8}}  + \frac{1}{2}\frac{\sigma_{ \lambda}^2 }{m^{3/2}T^{3/8}}
   \nonumber \\ & \quad \leq  O\Big{(}\frac{D_{\mathcal{W}}^2+G_{w}^2}{\sqrt{T}} +\frac{D_{\Lambda}^2}{T^{3/8}} +\frac{G_{\lambda}^2}{m^{1/2}T^{3/8}} \nonumber \\ & \quad \quad+\frac{\sigma_{\lambda}^2}{m^{3/2}T^{3/8}}+ \frac{\sigma_w^2+\Gamma}{m\sqrt{T} }\Big{)}.
\end{align}

\subsection{ADDITIONAL EXPERIMENTS}\label{extra_experiments}
\begin{figure*}[t]
\centering
\includegraphics[width=1\textwidth,  height=6cm, trim={0cm 0cm 0cm 0cm}]{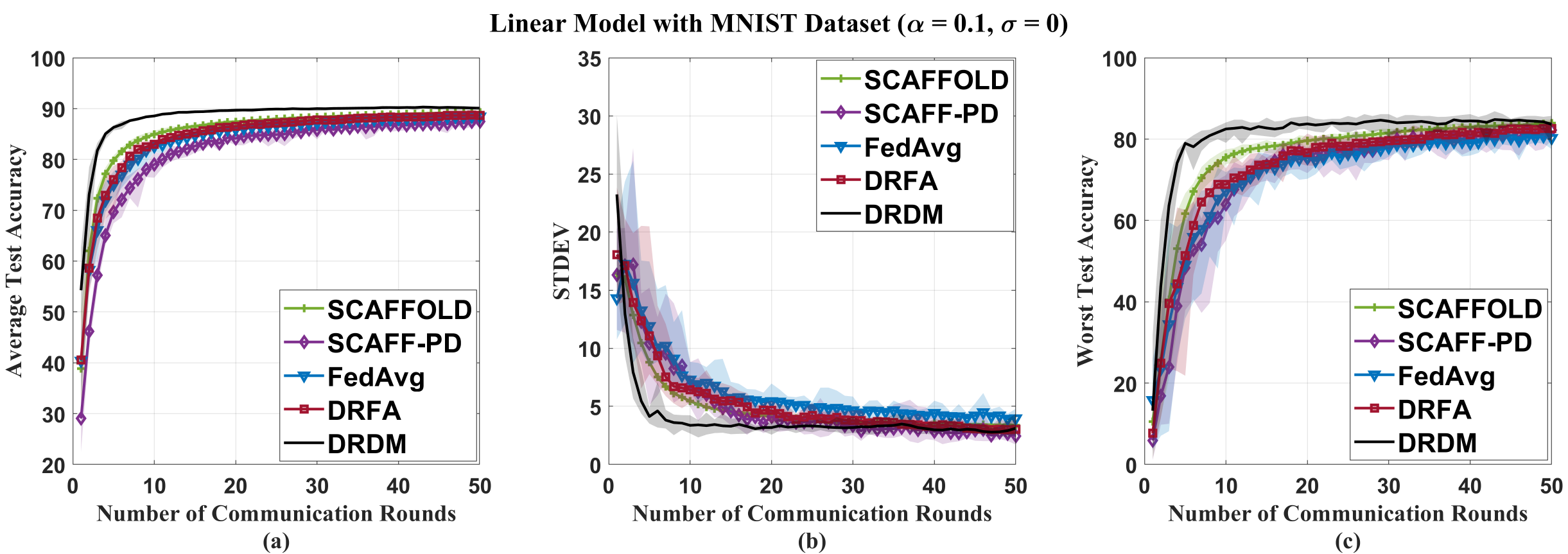}	
\caption{Results of \textit{DRDM} compared to the other baselines using Linear model with MNIST dataset with non-IID setting ($\alpha = 0.1$ and $\sigma = 0$). (a) Average test accuracy, (b) standard deviation values, and (c) worst test accuracy experienced by the different algorithms.}
\label{fig:linear_01_0}
\end{figure*} 
\begin{figure*}[t]
\centering
\includegraphics[width=1\textwidth,  height=6cm, trim={0cm 0cm 0cm 0cm}]{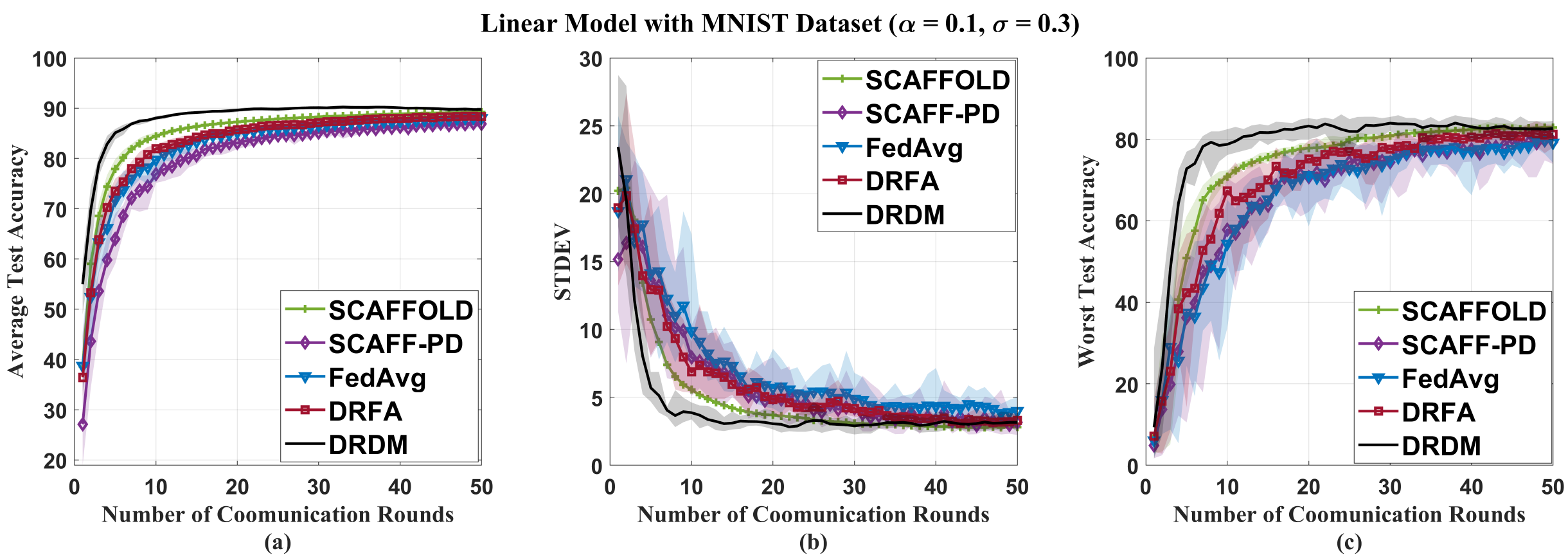}	
\caption{Results of \textit{DRDM} compared to the other baselines using Linear model with MNIST dataset with non-IID setting ($\alpha = 0.1$ and $\sigma = 0.3$). (a) Average test accuracy, (b) standard deviation values, and (c) worst test accuracy experienced by the different algorithms.}
\label{fig:linear_01_03}
\end{figure*} 
This section provides additional experimental results on the MNIST and Fashion-MNIST datasets to further evaluate \textit{DRDM} against baseline methods under various heterogeneity settings.   

Figures~\ref{fig:linear_01_0} and~\ref{fig:linear_01_03} present the comparison using a linear model with the MNIST dataset under non-IID conditions with fixed class heterogeneity $\alpha=0.1$ and varying dataset size levels ($\sigma=0$ and $\sigma=0.3$, respectively).  Additional results for other values of $\alpha$ and $\sigma$ are provided in Tables~\ref{table:linear_00} and~\ref{table:linear_01}. 
The experiments evaluate the average test accuracy, the standard deviation of the accuracy between clients, and the worst-case test accuracy. \textit{DRDM} consistently outperforms the baselines in terms of the average and worst-case test accuracy while maintaining competitive standard deviation between clients.

Tables~\ref{table:finetune_00} and~\ref{table:finetune_01} extend the evaluation to a different setup, where the last layer of the pre-trained ResNet18 model is fine-tuned on the Fashion-MNIST dataset. \textit{DRDM} demonstrates a strong performance across the varying levels of heterogeneity, highlighting its ability to deliver stable performance for different model and dataset complexities. Overall, these additional experiments reinforce the findings from Section \ref{experiments}, demonstrating the effectiveness and robustness of \textit{DRDM} across different datasets, model architectures, and heterogeneity levels.

\begin{table*}[ht]
\caption{Results of the different algorithms using Linear model and MNIST dataset, with $\sigma = 0$ and $\alpha \in \{0.1, 0.3, 0.5, 0.7\}$.}
\centering
\begin{tabular}{|
>{}c l|
>{}c 
>{}c 
>{}c 
>{}c |}
\hline
\multicolumn{2}{|c|}{Dataset} &
  \multicolumn{4}{c|}{MNIST} \\ \hline
\multicolumn{2}{|c|}{{($\alpha | \sigma = 0$)}} &
  \multicolumn{1}{P{2cm}|}{0.1} &
  \multicolumn{1}{P{2cm}|}{0.3} &
  \multicolumn{1}{P{2cm}|}{0.5} &
   \multicolumn{1}{P{2cm}|}{0.7} \\ \hline
\multicolumn{2}{|c|}{Model type} &
  \multicolumn{4}{c|}{Linear} \\ \hline
\multicolumn{1}{|c|}{} &
  \multicolumn{1}{l|}{{Average test accuracy}} &
  \multicolumn{1}{c|}{\textbf{90.31}} &
  \multicolumn{1}{c|}{\textbf{89.11}} &
  \multicolumn{1}{c|}{\textbf{90.20}} &
  \textbf{89.91} \\ \cline{2-6} 
\multicolumn{1}{|c|}{} &
  \multicolumn{1}{l|}{{Worst test accuracy}} &
  \multicolumn{1}{c|}{\textbf{84.81}} &
  \multicolumn{1}{c|}{\textbf{85.68}} &
  \multicolumn{1}{c|}{\textbf{86.36}} &
  \textbf{86.36} \\ \cline{2-6} 
\multicolumn{1}{|c|}{\multirow{-3}{*}{DRDM}} &
   \multicolumn{1}{l|}{{Standard deviation}} &
  \multicolumn{1}{c|}{2.96} &
  \multicolumn{1}{c|}{\textbf{2.29}} &
  \multicolumn{1}{c|}{\textbf{2.37}} &
  \textbf{2.14} \\ \hline
\multicolumn{1}{|c|}{} &
  \multicolumn{1}{l|}{{Average test accuracy}} &
  \multicolumn{1}{c|}{88.67} &
  \multicolumn{1}{c|}{88.07} &
  \multicolumn{1}{c|}{89.18} &
  88.13 \\ \cline{2-6} 
\multicolumn{1}{|c|}{} &
  \multicolumn{1}{l|}{{Worst test accuracy}} &
  \multicolumn{1}{c|}{82.60} &
  \multicolumn{1}{c|}{83.22} &
  \multicolumn{1}{c|}{83.86} &
  83.41 \\ \cline{2-6} 
\multicolumn{1}{|c|}{\multirow{-3}{*}{DRFA}} &
   \multicolumn{1}{l|}{{Standard deviation}} &
  \multicolumn{1}{c|}{3.02} &
  \multicolumn{1}{c|}{2.86} &
  \multicolumn{1}{c|}{2.80} &
  2.58 \\ \hline
\multicolumn{1}{|c|}{} &
  \multicolumn{1}{l|}{{Average test accuracy}} &
  \multicolumn{1}{c|}{89.32} &
  \multicolumn{1}{c|}{88.45} &
  \multicolumn{1}{c|}{89.15} &
  88.47 \\ \cline{2-6} 
\multicolumn{1}{|c|}{} &
  \multicolumn{1}{l|}{{Worst test accuracy}} &
  \multicolumn{1}{c|}{83.92} &
  \multicolumn{1}{c|}{84.10} &
  \multicolumn{1}{c|}{84.77} &
  83.67 \\ \cline{2-6} 
\multicolumn{1}{|c|}{\multirow{-3}{*}{SCAFFOLD}} &
   \multicolumn{1}{l|}{{Standard deviation}} &
  \multicolumn{1}{c|}{3.39} &
  \multicolumn{1}{c|}{2.61} &
  \multicolumn{1}{c|}{2.55} &
  2.49 \\ \hline
\multicolumn{1}{|c|}{} &
  \multicolumn{1}{l|}{{Average test accuracy}} &
  \multicolumn{1}{c|}{87.49} &
  \multicolumn{1}{c|}{86.38} &
  \multicolumn{1}{c|}{87.94} &
  87.68 \\ \cline{2-6} 
\multicolumn{1}{|c|}{} &
  \multicolumn{1}{l|}{{Worst test accuracy}} &
  \multicolumn{1}{c|}{82.23} &
  \multicolumn{1}{c|}{81.59} &
  \multicolumn{1}{c|}{83.22} &
  84.09 \\ \cline{2-6} 
\multicolumn{1}{|c|}{\multirow{-3}{*}{SCAFF-PD}} &
   \multicolumn{1}{l|}{{Standard deviation}} &
  \multicolumn{1}{c|}{\textbf{2.45}} &
  \multicolumn{1}{c|}{3.11} &
  \multicolumn{1}{c|}{2.58} &
  2.23 \\ \hline
\multicolumn{1}{|c|}{} &
  \multicolumn{1}{l|}{{Average test accuracy}} &
  \multicolumn{1}{c|}{88.46} &
  \multicolumn{1}{c|}{87.81} &
  \multicolumn{1}{c|}{88.86} &
  87.69 \\ \cline{2-6} 
\multicolumn{1}{|c|}{} &
  \multicolumn{1}{l|}{{Worst test accuracy}} &
  \multicolumn{1}{c|}{80.70} &
  \multicolumn{1}{c|}{82.96} &
  \multicolumn{1}{c|}{82.99} &
  84.09 \\ \cline{2-6} 
\multicolumn{1}{|c|}{\multirow{-3}{*}{Fed-Avg}} &
   \multicolumn{1}{l|}{{Standard deviation}} &
  \multicolumn{1}{c|}{3.93} &
  \multicolumn{1}{c|}{3.19} &
  \multicolumn{1}{c|}{2.97} &
  2.41 \\ \hline
\end{tabular}
\label{table:linear_00}
\end{table*}

\begin{table*}[ht]
\caption{Results of the different algorithms using Linear model and MNIST dataset, with $\alpha = 0.1$ and $\sigma \in \{0.3, 0.5, 0.7\}$.}
\centering
\begin{tabular}{|
>{}c l|
>{}c 
>{}c 
>{}c |}
\hline
\multicolumn{2}{|c|}{DATASET} &
  \multicolumn{3}{c|}{MNIST} \\ \hline
\multicolumn{2}{|c|}{{($\sigma | \alpha = 0.1$)}} &
  \multicolumn{1}{P{3cm}|}{0.3} &
  \multicolumn{1}{P{3cm}|}{0.5} &
  \multicolumn{1}{P{3cm}|}{0.7} \\ \hline
\multicolumn{2}{|c|}{Model type} &
  \multicolumn{3}{c|}{Linear} \\ \hline
\multicolumn{1}{|c|}{} &
  \multicolumn{1}{l|}{{Average test accuracy}} &
  \multicolumn{1}{c|}{\textbf{90.20}} &
  \multicolumn{1}{c|}{\textbf{90.05}} &
  \textbf{90.48} \\ \cline{2-5} 
\multicolumn{1}{|c|}{} &
   \multicolumn{1}{l|}{{Worst test accuracy}} &
  \multicolumn{1}{c|}{\textbf{84.00}} &
  \multicolumn{1}{c|}{\textbf{82.68}} &
  \textbf{83.07} \\ \cline{2-5} 
\multicolumn{1}{|c|}{\multirow{-3}{*}{DRDM}} &
  \multicolumn{1}{l|}{{Standard deviation}} &
  \multicolumn{1}{c|}{2.94} &
  \multicolumn{1}{c|}{3.66} &
  \textbf{3.30} \\ \hline
\multicolumn{1}{|c|}{} &
  \multicolumn{1}{l|}{{Average test accuracy}} &
  \multicolumn{1}{c|}{88.35} &
  \multicolumn{1}{c|}{88.74} &
  88.85 \\ \cline{2-5} 
\multicolumn{1}{|c|}{} &
   \multicolumn{1}{l|}{{Worst test accuracy}} &
  \multicolumn{1}{c|}{81.69} &
  \multicolumn{1}{c|}{81.37} &
  80.91 \\ \cline{2-5} 
\multicolumn{1}{|c|}{\multirow{-3}{*}{DRFA}} &
  \multicolumn{1}{l|}{{Standard deviation}} &
  \multicolumn{1}{c|}{3.15} &
  \multicolumn{1}{c|}{3.70} &
  4.41 \\ \hline
\multicolumn{1}{|c|}{} &
  \multicolumn{1}{l|}{{Average test accuracy}} &
  \multicolumn{1}{c|}{89.21} &
  \multicolumn{1}{c|}{89.34} &
  90.10 \\ \cline{2-5} 
\multicolumn{1}{|c|}{} &
   \multicolumn{1}{l|}{{Worst test accuracy}} &
  \multicolumn{1}{c|}{82.95} &
  \multicolumn{1}{c|}{80.18} &
  81.50 \\ \cline{2-5} 
\multicolumn{1}{|c|}{\multirow{-3}{*}{SCAFFOLD}} &
  \multicolumn{1}{l|}{{Standard deviation}} &
  \multicolumn{1}{c|}{\textbf{2.81}} &
  \multicolumn{1}{c|}{3.70} &
  4.00 \\ \hline
\multicolumn{1}{|c|}{} &
  \multicolumn{1}{l|}{{Average test accuracy}} &
  \multicolumn{1}{c|}{86.90} &
  \multicolumn{1}{c|}{86.59} &
  85.72 \\ \cline{2-5} 
\multicolumn{1}{|c|}{} &
   \multicolumn{1}{l|}{{Worst test accuracy}} &
  \multicolumn{1}{c|}{80.14} &
  \multicolumn{1}{c|}{80.89} &
  76.11 \\ \cline{2-5} 
\multicolumn{1}{|c|}{\multirow{-3}{*}{SCAFF-PD}} &
  \multicolumn{1}{l|}{{Standard deviation}} &
  \multicolumn{1}{c|}{2.91} &
  \multicolumn{1}{c|}{\textbf{3.02}} &
  4.18 \\ \hline
\multicolumn{1}{|c|}{} &
  \multicolumn{1}{l|}{{Average test accuracy}} &
  \multicolumn{1}{c|}{88.07} &
  \multicolumn{1}{c|}{88.47} &
  88.46 \\ \cline{2-5} 
\multicolumn{1}{|c|}{} &
   \multicolumn{1}{l|}{{Worst test accuracy}} &
  \multicolumn{1}{c|}{79.89} &
  \multicolumn{1}{c|}{80.35} &
  79.04 \\ \cline{2-5} 
\multicolumn{1}{|c|}{\multirow{-3}{*}{Fed-Avg}} &
  \multicolumn{1}{l|}{{Standard deviation}} &
  \multicolumn{1}{c|}{3.91} &
  \multicolumn{1}{c|}{4.29} &
  5.30 \\ \hline
\end{tabular}
\label{table:linear_01}
\end{table*}

\begin{table*}[ht]
\caption{Results of the different algorithms fine-tuning the last layer of ResNet18 pre-trained model, using Fashion-MNIST dataset, with $\sigma = 0$ and $\alpha \in \{0.1, 0.3, 0.5, 0.7, 0.9\}$.}
\centering
\begin{tabular}{|
>{}c l|
>{}c 
>{}c 
>{}c 
>{}c 
>{}c |}
\hline
\multicolumn{2}{|c|}{Dataset} &
  \multicolumn{5}{c|}{Fashion-MNIST} \\ \hline
\multicolumn{2}{|c|}{{($\alpha | \sigma = 0$)}} &
  \multicolumn{1}{P{2cm}|}{0.1} &
  \multicolumn{1}{P{2cm}|}{0.3} &
  \multicolumn{1}{P{2cm}|}{0.5} &
  \multicolumn{1}{P{2cm}|}{0.7} &
  \multicolumn{1}{P{2cm}|}{0.9} \\ \hline
\multicolumn{2}{|c|}{Model type} &
  \multicolumn{5}{c|}{Fine-tuning the last layer of pre-trained model ResNet18} \\ \hline
\multicolumn{1}{|c|}{} &
   \multicolumn{1}{l|}{{Average test accuracy}} &
  \multicolumn{1}{c|}{67.55} &
  \multicolumn{1}{c|}{\textbf{65.67}} &
  \multicolumn{1}{c|}{\textbf{65.15}} &
  \multicolumn{1}{c|}{\textbf{65.34}} &
  \textbf{66.31} \\ \cline{2-7} 
\multicolumn{1}{|c|}{} &
  \multicolumn{1}{l|}{{Worst test accuracy}} &
  \multicolumn{1}{c|}{\textbf{39.51}} &
  \multicolumn{1}{c|}{\textbf{44.78}} &
  \multicolumn{1}{c|}{\textbf{47.65}} &
  \multicolumn{1}{c|}{\textbf{51.04}} &
  51.21 \\ \cline{2-7} 
\multicolumn{1}{|c|}{\multirow{-3}{*}{DRDM}} &
  \multicolumn{1}{l|}{{Standard deviation}} &
  \multicolumn{1}{c|}{12.39} &
  \multicolumn{1}{c|}{\textbf{9.52}} &
  \multicolumn{1}{c|}{8.70} &
  \multicolumn{1}{c|}{\textbf{7.34}} &
  7.30 \\ \hline
\multicolumn{1}{|c|}{} &
   \multicolumn{1}{l|}{{Average test accuracy}}&
  \multicolumn{1}{c|}{67.43} &
  \multicolumn{1}{c|}{63.36} &
  \multicolumn{1}{c|}{63.37} &
  \multicolumn{1}{c|}{62.82} &
  65.82 \\ \cline{2-7} 
\multicolumn{1}{|c|}{} &
  \multicolumn{1}{l|}{{Worst test accuracy}} &
  \multicolumn{1}{c|}{37.41} &
  \multicolumn{1}{c|}{39.38} &
  \multicolumn{1}{c|}{45.01} &
  \multicolumn{1}{c|}{47.88} &
  \textbf{51.80} \\ \cline{2-7} 
\multicolumn{1}{|c|}{\multirow{-3}{*}{DRFA}} &
  \multicolumn{1}{l|}{{Standard deviation}} &
  \multicolumn{1}{c|}{12.88} &
  \multicolumn{1}{c|}{9.94} &
  \multicolumn{1}{c|}{8.53} &
  \multicolumn{1}{c|}{7.74} &
  \textbf{6.66} \\ \hline
\multicolumn{1}{|c|}{} &
    \multicolumn{1}{l|}{{Average test accuracy}} &
  \multicolumn{1}{c|}{\textbf{70.58}} &
  \multicolumn{1}{c|}{65.17} &
  \multicolumn{1}{c|}{64.83} &
  \multicolumn{1}{c|}{64.99} &
  65.34 \\ \cline{2-7} 
\multicolumn{1}{|c|}{} &
  \multicolumn{1}{l|}{{Worst test accuracy}} &
  \multicolumn{1}{c|}{38.01} &
  \multicolumn{1}{c|}{40.77} &
  \multicolumn{1}{c|}{47.05} &
  \multicolumn{1}{c|}{50.00} &
  51.07 \\ \cline{2-7} 
\multicolumn{1}{|c|}{\multirow{-3}{*}{SCAFFOLD}} &
  \multicolumn{1}{l|}{{Standard deviation}} &
  \multicolumn{1}{c|}{\textbf{12.35}} &
  \multicolumn{1}{c|}{9.88} &
  \multicolumn{1}{c|}{\textbf{8.494}} &
  \multicolumn{1}{c|}{7.59} &
  6.66 \\ \hline
\multicolumn{1}{|c|}{} &
   \multicolumn{1}{l|}{{Average test accuracy}} &
  \multicolumn{1}{c|}{59.41} &
  \multicolumn{1}{c|}{61.59} &
  \multicolumn{1}{c|}{60.51} &
  \multicolumn{1}{c|}{62.28} &
  62.17 \\ \cline{2-7} 
\multicolumn{1}{|c|}{} &
  \multicolumn{1}{l|}{{Worst test accuracy}} &
  \multicolumn{1}{c|}{34.44} &
  \multicolumn{1}{c|}{39.82} &
  \multicolumn{1}{c|}{42.90} &
  \multicolumn{1}{c|}{43.70} &
  48.64 \\ \cline{2-7} 
\multicolumn{1}{|c|}{\multirow{-3}{*}{SCAFF-PD}} &
  \multicolumn{1}{l|}{{Standard deviation}} &
  \multicolumn{1}{c|}{12.92} &
  \multicolumn{1}{c|}{10.44} &
  \multicolumn{1}{c|}{8.94} &
  \multicolumn{1}{c|}{8.01} &
  7.21 \\ \hline
\multicolumn{1}{|c|}{} &
    \multicolumn{1}{l|}{{Average test accuracy}} &
  \multicolumn{1}{c|}{63.07} &
  \multicolumn{1}{c|}{59.68} &
  \multicolumn{1}{c|}{58.91} &
  \multicolumn{1}{c|}{61.51} &
  61.95 \\ \cline{2-7} 
\multicolumn{1}{|c|}{} &
  \multicolumn{1}{l|}{{Worst test accuracy}} &
  \multicolumn{1}{c|}{28.43} &
  \multicolumn{1}{c|}{34.90} &
  \multicolumn{1}{c|}{39.44} &
  \multicolumn{1}{c|}{46.06} &
  46.53 \\ \cline{2-7} 
\multicolumn{1}{|c|}{\multirow{-3}{*}{Fed-Avg}} &
  \multicolumn{1}{l|}{{Standard deviation}} &
  \multicolumn{1}{c|}{15.26} &
  \multicolumn{1}{c|}{10.93} &
  \multicolumn{1}{c|}{9.55} &
  \multicolumn{1}{c|}{8.02} &
  7.27 \\ \hline
\end{tabular}
\label{table:finetune_00}
\end{table*}

\begin{table*}[ht]
\caption{Results of the different algorithms fine-tuning the last layer of ResNet18 pre-trained model, using Fashion-MNIST dataset, with $\alpha = 0.1$ and $\sigma \in \{0.3, 0.5, 0.7\}$.}
\centering
\begin{tabular}{|
>{}c l|
>{}c 
>{}c 
>{}c |}
\hline
\multicolumn{2}{|c|}{Dataset} &
  \multicolumn{3}{c|}{Fashion-MNIST} \\ \hline
\multicolumn{2}{|c|}{{($\sigma | \alpha = 0.1$)}} &
  \multicolumn{1}{P{3.5cm}|}{0.3} &
  \multicolumn{1}{P{2cm}|}{0.5} &
  \multicolumn{1}{P{2cm}|}{0.7} \\ \hline
\multicolumn{2}{|c|}{Model type} &
  \multicolumn{3}{c|}{Fine-tuning the last layer of pre-trained model ResNet18} \\ \hline
\multicolumn{1}{|c|}{} &
   \multicolumn{1}{l|}{{Average test accuracy}} &
  \multicolumn{1}{c|}{\textbf{67.89}} &
  \multicolumn{1}{c|}{68.85} &
  67.66 \\ \cline{2-5} 
\multicolumn{1}{|c|}{} &
  \multicolumn{1}{l|}{{Worst test accuracy}} &
  \multicolumn{1}{c|}{\textbf{40.60}} &
  \multicolumn{1}{c|}{\textbf{43.41}} &
  \textbf{40.62} \\ \cline{2-5} 
\multicolumn{1}{|c|}{\multirow{-3}{*}{DRDM}} &
   \multicolumn{1}{l|}{{Standard deviation}} &
  \multicolumn{1}{c|}{12.37} &
  \multicolumn{1}{c|}{11.32} &
  12.77 \\ \hline
\multicolumn{1}{|c|}{} &
    \multicolumn{1}{l|}{{Average test accuracy}} &
  \multicolumn{1}{c|}{62.99} &
  \multicolumn{1}{c|}{65.57} &
  63.81 \\ \cline{2-5} 
\multicolumn{1}{|c|}{} &
  \multicolumn{1}{l|}{{Worst test accuracy}} &
  \multicolumn{1}{c|}{32.47} &
  \multicolumn{1}{c|}{37.27} &
  30.60 \\ \cline{2-5} 
\multicolumn{1}{|c|}{\multirow{-3}{*}{DRFA}} &
   \multicolumn{1}{l|}{{Standard deviation}} &
  \multicolumn{1}{c|}{13.55} &
  \multicolumn{1}{c|}{12.74} &
  13.43 \\ \hline
\multicolumn{1}{|c|}{} &
    \multicolumn{1}{l|}{{Average test accuracy}} &
  \multicolumn{1}{c|}{67.11} &
  \multicolumn{1}{c|}{\textbf{69.03}} &
  \textbf{67.80} \\ \cline{2-5} 
\multicolumn{1}{|c|}{} &
  \multicolumn{1}{l|}{{Worst test accuracy}} &
  \multicolumn{1}{c|}{36.11} &
  \multicolumn{1}{c|}{39.16} &
  34.84 \\ \cline{2-5} 
\multicolumn{1}{|c|}{\multirow{-3}{*}{SCAFFOLD}} &
   \multicolumn{1}{l|}{{Standard deviation}} &
  \multicolumn{1}{c|}{13.40} &
  \multicolumn{1}{c|}{12.39} &
  12.49 \\ \hline
\multicolumn{1}{|c|}{} &
    \multicolumn{1}{l|}{{Average test accuracy}} &
  \multicolumn{1}{c|}{61.24} &
  \multicolumn{1}{c|}{61.52} &
  62.08 \\ \cline{2-5} 
\multicolumn{1}{|c|}{} &
  \multicolumn{1}{l|}{{Worst test accuracy}} &
  \multicolumn{1}{c|}{36.30} &
  \multicolumn{1}{c|}{35.74} &
  37.90 \\ \cline{2-5} 
\multicolumn{1}{|c|}{\multirow{-3}{*}{SCAFF-PD}} &
   \multicolumn{1}{l|}{{Standard deviation}} &
  \multicolumn{1}{c|}{\textbf{11.38}} &
  \multicolumn{1}{c|}{\textbf{10.77}} &
  \textbf{11.98} \\ \hline
\multicolumn{1}{|c|}{} &
    \multicolumn{1}{l|}{{Average test accuracy}} &
  \multicolumn{1}{c|}{59.27} &
  \multicolumn{1}{c|}{61.80} &
  60.24 \\ \cline{2-5} 
\multicolumn{1}{|c|}{} &
  \multicolumn{1}{l|}{{Worst test accuracy}} &
  \multicolumn{1}{c|}{24.92} &
  \multicolumn{1}{c|}{27.66} &
  13.96 \\ \cline{2-5} 
\multicolumn{1}{|c|}{\multirow{-3}{*}{Fed-Avg}} &
  \multicolumn{1}{l|}{{Standard deviation}} &
  \multicolumn{1}{c|}{15.84} &
  \multicolumn{1}{c|}{14.85} &
  16.50 \\ \hline
\end{tabular}
\label{table:finetune_01}
\end{table*}

\end{document}